\documentclass{article}

\usepackage[nonatbib, final]{nips_2016_arxiv}

\usepackage[utf8]{inputenc} % allow utf-8 input
\usepackage[T1]{fontenc}    % use 8-bit T1 fonts
\usepackage{hyperref}       % hyperlinks
\hypersetup{
	colorlinks   = true, %Colours links instead of ugly boxes
	urlcolor     = blue, %Colour for external hyperlinks
	linkcolor    = blue, %Colour of internal links
	citecolor   = blue %Colour of citations
}
\usepackage{url}            % simple URL typesetting
\usepackage{booktabs}       % professional-quality tables
\usepackage{amsfonts}       % blackboard math symbols
\usepackage{nicefrac}       % compact symbols for 1/2, etc.
\usepackage{microtype}      % microtypography

\usepackage{graphicx} 
\usepackage{subfigure}

\usepackage{epstopdf}

\usepackage{amsmath}
\usepackage{amsthm}
\usepackage{amssymb}
\usepackage{dsfont}
\usepackage{enumerate}
\usepackage{algorithm}
\usepackage{algorithmic}

\usepackage{ifthen}
\usepackage[usenames]{color}
\usepackage[usenames,dvipsnames]{xcolor}

\usepackage{accents}
\makeatletter
\def\wideubar{\underaccent{{\cc@style\underline{\mskip10mu}}}}
\def\Wideubar{\underaccent{{\cc@style\underline{\mskip8mu}}}}
\makeatother

\makeatletter
\def\widebar{\accentset{{\cc@style\underline{\mskip10mu}}}}
\def\Widebar{\accentset{{\cc@style\underline{\mskip8mu}}}}
\makeatother

\newtheorem{definition}{Definition}
\newtheorem{lemma}{Lemma}
\newtheorem{theorem}{Theorem}

\newtheorem{corollary}{Corollary}

\newtheorem{assumption}{Assumption}

\newcommand{\R}{{\mathbb R}}
\newcommand{\Z}{{\mathbb Z}}

\newcommand{\E}{{\mathbb E}}

\newcommand{\A}{{\mathcal A}}

\newcommand{\G}{{\mathcal G}}
\newcommand{\calF}{{\mathcal F}}
\newcommand{\calH}{{\mathcal H}}
\newcommand{\calE}{{\mathcal E}}
\renewcommand{\|}{\Vert}
\newcommand{\OPT}{\mathsf{OPT}}

\newcommand{\reg}{\mathsf{Reg}}
\newcommand{\supp}{\mathsf{supp}}
\newcommand{\bad}{{\mathrm{B}}}

\newcommand{\argmax}{\operatornamewithlimits{argmax}}

\newcommand{\dx}{\mathop{}\!\mathrm{d}x}
\newcommand{\dFx}{\mathop{}\!\mathrm{d}F(x)}

\newcommand{\compilehidecomments}{false}

\ifthenelse{ \equal{\compilehidecomments}{true} }{%
	\newcommand{\weic}[1]{}
	\newcommand{\weih}[1]{}
}{
\newcommand{\weic}[1]{{\color{blue!50!black}  [\text{WeiC:} #1]}}
\newcommand{\weih}[1]{{\color{brown!60!black} [\text{WeiH:} #1]}}

}

\allowdisplaybreaks

\title{Combinatorial Multi-Armed Bandit with General Reward Functions
	}

\author{
  Wei Chen\thanks{Microsoft Research, email: \texttt{weic@microsoft.com}. The authors are listed in alphabetical order.}
  \And
  Wei Hu\thanks{Princeton University, email: \texttt{huwei@cs.princeton.edu}.}
  \And
  Fu Li\thanks{The University of Texas at Austin, email: \texttt{fuli.theory.research@gmail.com}.}
  \And
  Jian Li\thanks{Tsinghua University, email: \texttt{lapordge@gmail.com}.}
  \And
  Yu Liu\thanks{Tsinghua University, email: \texttt{liuyujyyz@gmail.com}.}
  \And
  Pinyan Lu\thanks{Shanghai University of Finance and Economics, email: \texttt{lu.pinyan@mail.shufe.edu.cn}.}
}

\begin{document}

\maketitle

\begin{abstract}
In this paper, we study the stochastic combinatorial multi-armed bandit (CMAB) framework that allows a general nonlinear reward function, whose expected value may not depend only on the means of the input random variables but possibly on the entire distributions of these variables. Our framework enables a much larger class of reward functions such as the $\max()$ function and nonlinear utility functions. Existing techniques relying on accurate estimations of the means of random variables, such as the upper confidence bound (UCB) technique, do not work directly on these functions. We propose a new algorithm called {\em stochastically dominant confidence bound (SDCB)}, which estimates the distributions of underlying random variables and their stochastically dominant confidence bounds. We prove that SDCB can achieve $O(\log T)$ distribution-dependent regret and $\tilde{O}(\sqrt{T})$ distribution-independent regret, where $T$ is the time horizon. We apply our results to the $K$-MAX problem and expected utility maximization problems. In particular, for $K$-MAX, we provide the first polynomial-time approximation scheme (PTAS) for its offline problem, and give the first $\tilde{O} (\sqrt T)$ bound on the $(1-\epsilon)$-approximation regret of its online problem, for any $\epsilon>0$.
%  In this paper, we study the stochastic combinatorial multi-armed bandit (CMAB) framework that
%  allows a general nonlinear reward function, whose expected value 
%  may not depend only on the means of the input random variables but possibly on the entire
%  distributions of these variables.	
%  Our framework enables a much larger class of reward functions such as the  $\max()$ function
%  and nonlinear utility functions.
%  Existing techniques relying on accurate estimations of the means of random variables, such
%  as the upper confidence bound (UCB) technique, do not work directly on these functions.
%  We propose a new algorithm called {\em stochastically dominant confidence bound (SDCB)}, which
%  estimates the distributions of underlying random variables and their stochastically dominant
%  confidence bounds. 
%  We prove that SDCB can achieve $O(\log T)$ distribution-dependent regret and $\tilde{O}(\sqrt{T})$ distribution-independent regret, where $T$ is the time horizon.
%We apply our results to the $K$-MAX problem and expected utility maximization problems.
%In particular, for $K$-MAX, we provide the first polynomial-time approximation scheme (PTAS) for its offline problem, and give the first
% $\tilde{O} (\sqrt T)$ bound
%	on the $(1-\epsilon)$-approximation regret of its online problem, for any $\epsilon>0$.
\end{abstract}

\setcounter{footnote}{0}

\section{Introduction}

Stochastic multi-armed bandit (MAB) is a classical online learning problem typically specified
as a player against $m$ machines or arms.
Each arm, when pulled, generates a random reward following an unknown distribution.
The task of the player is to select one arm to pull in each round based on the historical
	rewards she collected, and the goal is to collect
cumulative reward over multiple rounds as much as possible.
In this paper, unless otherwise specified, we use MAB to refer to stochastic MAB.

%\footnote{In this paper
%	we do not consider {\em adversarial} MAB, where
%	the reward is determined by an adversary.
%	Henceforth MAB
%	always refers to the stochastic version, unless we explicitly specify otherwise.
%	}
%The player critically relies on the historical rewards she collected in previous rounds
%of play to decide which arm to select in the next round.

MAB problem demonstrates the key tradeoff between exploration and
exploitation: whether the player should stick to the choice that performs the best so far,
or should try some less explored alternatives that may provide better rewards.
The performance measure of an MAB strategy is
its cumulative {\em regret}, which is defined as the difference between the cumulative reward
obtained by always playing the arm with the largest expected reward % when the reward
	%distributions of the arms are known, 
	and the cumulative reward achieved by the learning strategy.
MAB and its variants have been extensively studied in the literature,
with classical results such as tight $\Theta (\log T)$ distribution-dependent and $\Theta(\sqrt T)$ distribution-independent upper and lower bounds on the regret in $T$ rounds \cite{lai1985asymptotically, auer2002finite, audibert2009minimax}.

% $\Omega(\log T)$ regret lower bound, 
%where $T$ is the number of rounds played \cite{lai1985asymptotically}, 
%and an upper confidence bound (UCB) based solution that achieves $O(\log T)$ regret
%uniformly over all $T$ \cite{auer2002finite}.

An important extension to the classical MAB problem is
combinatorial multi-armed bandit (CMAB).
In CMAB, the player selects not just one arm in each round, but a subset of
arms or a combinatorial object in general, referred to as a super arm, 
which collectively provides a random reward
to the player.
The reward depends on the outcomes from the selected arms. % unknown distributions of the underlying arms.
The player may observe partial feedbacks from the selected arms to help her
in decision making.
CMAB has wide applications in online advertising, online recommendation, 
wireless routing, dynamic channel allocations, etc., because in all these settings
the action unit is a combinatorial object (e.g. a set of advertisements, a set of
recommended items, a route in a wireless network, and an allocation between channels and
users), and the reward depends on unknown stochastic behaviors (e.g. users' click 
through behaviors, wireless transmission quality, etc.).
Therefore CMAB has attracted a lot of attention in online learning research in recent
years
\cite{Yi2012, chen2016cmab, lin2014combinatorial, GMM14, chen2014pure, KWAEE14, Kveton15, Kveton15cascade, lin2015online, Combes2015}.

Most of these studies focus on linear reward functions, for which the expected reward for playing a super
arm is a linear combination of the expected outcomes from the constituent base arms.
Even for studies that do generalize to non-linear reward functions, they typically still assume
that the expected reward for choosing a super arm is a function of the expected outcomes from the constituent base arms in this super arm
\cite{chen2016cmab,Kveton15cascade}.
However, many natural reward functions do not satisfy this property.
For example, for the function $\max()$, which takes a group of 
variables and outputs the
maximum one among them, its expectation depends on the full distributions of the input random
variables, not just their means.
Function $\max()$ and its variants underly many applications.
As an illustrative example, we consider the following scenario in auctions:
the auctioneer is repeatedly selling an item to $m$ bidders;
in each round the auctioneer selects $K$ bidders to bid; each of the $K$ bidders independently draws
her bid from her private valuation distribution and submits the bid; the auctioneer uses
the first-price auction to determine the winner and collects the largest bid as the 
payment.\footnote{We understand that the first-price auction is not truthful, but this
	example is only for illustrative purpose for the $\max()$ function.}
The goal of the auctioneer is %to gradually learn bidders' valuation distributions and
to gain as high cumulative payments as possible.
We refer to this problem as the $K$-MAX bandit problem, which cannot be effectively solved in the existing
CMAB framework.

Beyond the $K$-MAX problem, many expected utility maximization (EUM) problems are
	studied in stochastic optimization literature~\cite{yu2016maximizing,LiD11,li2013stochastic,bhalgat2014utility}.
The problem can be formulated as maximizing $\E[u(\sum_{i\in S} X_i)]$ 
among all
	 feasible sets $S$, where $X_i$'s are independent 
	random variables and $u(\cdot)$ is a utility function.
For example, $X_i$ could be the random delay of edge $e_i$ in a routing graph, 
	$S$ is a routing path in the graph, 
	and the objective is maximizing the utility obtained from any routing path,
	and typically the shorter the delay, the larger the utility.
The utility function $u(\cdot)$ is typically nonlinear to model risk-averse or risk-prone
	behaviors of users (e.g. a concave utility function is often used to model risk-averse
	behaviors).
The non-linear utility function makes the objective function much more complicated: in particular,
	it is no longer a function of the means of the underlying random variables $X_i$'s.
When the distributions of $X_i$'s are unknown, we can turn EUM into an online learning problem
	where the distributions of $X_i$'s need to be learned over time from online feedbacks, and we
	want to maximize the cumulative reward in the learning process.
Again, this is not covered by the existing CMAB framework since only learning the means of $X_i$'s
	is not enough.
	
In this paper, we generalize the existing CMAB framework with semi-bandit feedbacks to
handle general reward functions, where the expected reward for playing a super arm may depend more than
just the means of the base arms, and the outcome distribution of a base arm
can be arbitrary.
This generalization is non-trivial, because almost all previous works on CMAB rely on estimating the expected outcomes from base arms, %which can be made accurate using standard tools such as Chernoff bound.
while in our case, we need an estimation method and an analytical tool to deal with the whole distribution,
not just its mean.
To this end, 
%First, for discrete random variables with known finite supports,
we turn the problem into estimating the
	cumulative distribution function (CDF) of each arm's outcome distribution. %, which makes it much closer to the existing CMAB framework.
We use {\em stochastically dominant confidence bound (SDCB)} to obtain a distribution that stochastically dominates the true distribution with high probability,
and hence we also name our algorithm \texttt{SDCB}.
We are able to show $O(\log T)$ distribution-dependent and $\tilde{O}(\sqrt{T})$ 
distribution-independent regret bounds in $T$ rounds. %, where $T$ is the time horizon.
Furthermore, we propose a more efficient algorithm called \texttt{Lazy-SDCB}, which first executes a discretization step and then applies \texttt{SDCB} on the discretized problem.
%Second, we handle general bounded random variables by a discretization technique, and we also show
%$O(\log T)$ distribution-dependent\footnote{The $O(\log T)$ distribution-dependent bound for general variables only holds for sufficiently large $T$.}
We show that \texttt{Lazy-SDCB} also achieves $\tilde{O}(\sqrt{T})$ distribution-independent regret bound.
Our regret bounds are tight with respect to their dependencies on $T$ (up to a logarithmic factor for distribution-independent bounds). %\footnote{There are $\Omega(\log T)$ distribution-dependent and $\Omega(\sqrt T)$ distribution-independent lower bounds on the classical MAB problem, a special instance of CMAB.}
To make our scheme work, we make a few reasonable assumptions, including boundedness, monotonicity and Lipschitz-continuity\footnote{The Lipschitz-continuity assumption is only made for \texttt{Lazy-SDCB}. See Section~\ref{sec:discretization}.} of the reward function, and independence among base arms.
We apply our algorithms to the $K$-MAX and EUM problems, and provide efficient solutions with concrete regret bounds.
Along the way, we also provide the first polynomial time approximation scheme (PTAS) for the offline
	$K$-MAX problem, which is formulated as maximizing $\E[\max_{i\in S}X_i]$ subject to a cardinality constraint $|S|\le K$, where $X_i$'s are independent nonnegative random variables.

To summarize, our contributions include: (a) generalizing the CMAB framework to allow a
general reward function whose expectation may depend on the entire distributions of the input random variables;
(b) proposing the SDCB algorithm to achieve efficient learning in this framework with
near-optimal regret bounds, even for arbitrary outcome distributions;
(c) giving the first PTAS for the offline $K$-MAX problem.
Our general framework treats any offline stochastic optimization algorithm as an oracle, and effectively integrates
	it into the online learning framework.

\paragraph{Related Work.}
As already mentioned, most relevant to our work are studies on CMAB frameworks, 
among which \cite{Yi2012,KWAEE14,Kveton15,Combes2015} focus on linear reward functions
while \cite{chen2016cmab,Kveton15cascade} look into non-linear reward functions.
In particular, Chen et al. \cite{chen2016cmab} look at general non-linear reward functions and Kveton et al.
\cite{Kveton15cascade} consider specific non-linear reward functions in a conjunctive or
disjunctive form, but both papers require that the expected reward of playing a super arm
is determined by the expected outcomes from base arms.

The only work in combinatorial bandits 
we are aware of that does not require the above assumption on the expected reward
is \cite{GMM14}, which is based on a general Thompson sampling framework.
However, they assume that the joint distribution of base arm outcomes is from a known
parametric family within known likelihood function and only the parameters are unknown.
They also assume the parameter space to be finite.
In contrast, our general case is non-parametric, where we allow arbitrary bounded distributions.
Although in our known finite support case the distribution can be parametrized by probabilities
on all supported points, our parameter space is continuous.
Moreover,
it is unclear
how to efficiently compute posteriors in their algorithm, and their regret bounds
depend on complicated problem-dependent coefficients which may be very large
for many combinatorial problems.
They also provide a result on the $K$-MAX problem, but they only consider Bernoulli outcomes from
base arms, much simpler than our case where general distributions are allowed.

There are extensive studies on the classical MAB problem, for which we refer
to a survey by Bubeck and Cesa-Bianchi \cite{BubeckC12}. 
There are also some studies on adversarial combinatorial bandits, e.g. \cite{Streeter2008,cesa2012combinatorial}.
Although it bears conceptual similarities with stochastic CMAB, the techniques used are different.

Expected utility maximization (EUM) encompasses a large class of stochastic optimization problems and has
	been well studied (e.g. \cite{yu2016maximizing,LiD11,li2013stochastic,bhalgat2014utility}).
To the best of our knowledge, we are the first to study the online learning version of these problems, and we provide a general
	solution to systematically address all these problems as long as there is an available offline (approximation) algorithm.
The $K$-MAX problem may be traced back to \cite{goel2006asking}, where Goel et al. provide a constant approximation algorithm
	to a generalized version in which the objective is to choose a subset $S$ of cost at most $K$ and
	maximize the expectation of a certain knapsack profit.

\section{Setup and Notation} \label{sec:problem}

\paragraph{Problem Formulation.}
We model a combinatorial multi-armed bandit (CMAB) problem as a tuple $(E, \calF, D, R)$, where $E = [m] = \{1, 2, \ldots, m\}$ is a set of $m$ (base) arms, $\calF \subseteq 2^{E}$ is a set of subsets of $E$, $D$ is a probability distribution over $[0, 1]^m$, and $R$ is a reward function defined on $[0, 1]^m \times \calF$.
The arms produce stochastic outcomes $X = (X_1, X_2, \ldots, X_m)$ drawn from distribution $D$, where the $i$-th entry $X_i$ is the outcome from the $i$-th arm. Each feasible subset of arms $S \in \calF$ is called a \emph{super arm}.
Under a realization of outcomes $x = (x_1, \ldots, x_m)$, the player receives a reward $R(x, S)$ when she chooses the super arm $S$ to play. Without loss of generality, we assume the reward value to be nonnegative.
Let $K = \max_{S \in \calF} |S|$ be the maximum size of any super arm.

Let $X^{(1)}, X^{(2)}, \ldots$ be an i.i.d. sequence of random vectors drawn from $D$, where $X^{(t)} = (X^{(t)}_1, \ldots, X^{(t)}_m)$ is the outcome vector generated in the $t$-th round. In the $t$-th round, the player chooses a super arm $S_t \in \calF$ to play, and then the outcomes from all arms in $S_t$, i.e., $\{X^{(t)}_i \, | \, i\in S_t \}$, are revealed to the player.
According to the definition of the reward function, the reward value in the $t$-th round is $R(X^{(t)}, S_t)$.
The expected reward for choosing a super arm $S$ in any round %, when the outcome vector is drawn from $D$,
is denoted by $r_D(S) = \E_{X\sim D}[R(X, S)]$.

We also assume that for a fixed super arm $S \in \calF$, the reward $R(x, S)$ only depends on the revealed outcomes $x_S = (x_i)_{i\in S}$. Therefore, we can alternatively express $R(x, S)$ as $R_S(x_S)$, where $R_S$ is a function defined on $[0, 1]^{S}$.\footnote{$[0, 1]^S$ is isomorphic to $[0, 1]^{|S|}$; the coordinates in $[0, 1]^S$ are indexed by elements in $S$.}

A learning algorithm $\A$ for the CMAB problem selects which super arm to play in each round based on the revealed outcomes in all previous rounds. Let $S^\A_t$ be the super arm selected by $\A$ in the $t$-th round.\footnote{Note that $S^\A_t$ may be random due to the random outcomes in previous rounds and the possible randomness used by $\A$.}
The goal is to maximize the expected cumulative reward in $T$ rounds, which is $\E\left[ \sum_{t=1}^T R(X^{(t)}, S^\A_t) \right] =  \sum_{t=1}^T \E \left[ r_D(S^\A_t) \right]$. 
Note that when the underlying distribution $D$ is known, the optimal algorithm $\A^*$ chooses the optimal super arm $S^* = \argmax_{S \in \calF} \{r_D(S)\}$ in every round.
The quality of an algorithm $\A$ is measured by its \emph{regret} in $T$ rounds, which is the difference between the expected cumulative reward of the optimal algorithm $\A^*$ and that of $\A$:
\[
\reg_D^\A(T) = T \cdot r_D(S^*) - \sum_{t=1}^T \E \left[r_D(S^\A_t) \right].
\]

For some CMAB problem instances, the optimal super arm $S^*$ may be computationally hard to find even when the distribution $D$ is known, but efficient approximation algorithms may exist, i.e., an $\alpha$-approximate ($0 < \alpha \le 1$) solution $S' \in \calF$ which satisfies $r_D(S') \ge \alpha \cdot \max_{S \in \calF} \{r_D(S)\}$ can be efficiently found given $D$ as input.
We will provide the exact formulation of our requirement on such an \emph{$\alpha$-approximation computation oracle}
	shortly.
In such cases, it is not fair to compare a CMAB algorithm $\A$ with the optimal algorithm $\A^*$ which always chooses the optimal super arm $S^*$. Instead, we define the \emph{$\alpha$-approximation regret} of an algorithm $\A$ as
\[
\reg_{D, \alpha}^\A(T) = T \cdot \alpha \cdot r_D(S^*) - \sum_{t=1}^T \E \left[r_D(S^\A_t) \right].
\]

As mentioned, almost all previous work on CMAB requires that the expected reward $r_D(S)$ of a super arm $S$ depends only on the expectation vector $\mu = (\mu_1, \ldots, \mu_m)$ of outcomes, where $\mu_i = \E_{X\sim D}[X_i]$.
%We refer to this condition as the \emph{simple expected reward property}.
This is a strong restriction that cannot be satisfied by a general nonlinear function $R_S$ and a general distribution $D$. The main motivation of this work is to remove this restriction.

\paragraph{Assumptions.}
Throughout this paper, we make several assumptions on the outcome distribution $D$ and the reward function $R$.

\begin{assumption}[Independent outcomes from arms] \label{assum:indep}
	The outcomes from all $m$ arms are mutually independent, i.e., for $X \sim D$, $X_1, X_2, \ldots, X_m$ are mutually independent.
	We write $D$ as $D = D_1 \times D_2 \times \cdots \times D_m$, where $D_i$ is the distribution of $X_i$.
\end{assumption}
We remark that the above independence assumption is also made for past studies on the offline EUM and K-MAX problems \cite{yu2016maximizing,LiD11,li2013stochastic,bhalgat2014utility,goel2006asking}, so it is not an extra assumption for the online learning case.

\begin{assumption}[Bounded reward value] \label{assum:bounded}
	There exists $M>0$ such that for any $x\in[0, 1]^m$ and any $S \in \calF$, we have $0\le R(x, S) \le M$.
\end{assumption}

\begin{assumption}[Monotone reward function] \label{assum:monotone}
	If two vectors $x, x'\in [0, 1]^m$ satisfy $x_i \le x_i'$ ($\forall i\in[m]$), then for any $S \in \calF$, we have $R(x, S) \le R(x', S)$.
\end{assumption}

%\begin{assumption}[Lipschitz-continuous reward function] \label{assum:lipschitz}
%	There exists $C>0$ such that for any $S \in \calF$ and any $x, x'\in [0, 1]^m$, we have $\left| R(x, S) - R(x', S) \right| \le C\|x_S - x_S'\|_1$,
%	where $\| x_S - x_S'\|_1 = \sum_{i\in S} |x_i - x_i'|$.
%\end{assumption}

\paragraph{Computation Oracle for Discrete Distributions with Finite Supports.}
We require that there exists an $\alpha$-approximation computation oracle ($0< \alpha \le 1$) for maximizing $r_D(S)$, when each $D_i$ ($i\in[m]$) has a \emph{finite support}.
In this case, $D_i$ can be fully described by a finite set of numbers (i.e., its support $\{v_{i, 1}, v_{i, 2}, \ldots, v_{i, s_i}\}$ and the values of its cumulative distribution function (CDF) $F_i$ on the supported points: $F_i(v_{i, j}) = \Pr_{X_i \sim D_i}\left[X_i \le v_{i, j}\right] (j\in[s_i])$).
The oracle takes such a representation of $D$ as input, and can output a super arm $S' = \mathsf{Oracle}(D) \in \calF$ such that $r_D(S') \ge \alpha \cdot \max_{S \in \calF} \{ r_D(S)\}$.

%we suppose the support of $D_i$ is $\supp(D_i) = \{v_{i, 1}, v_{i, 2}, \ldots, v_{i, s_i}\}$ ($i\in[m]$).
%Then $D_i$ can be fully described by a set of its cumulative distribution function (CDF) values $F^D_{i, j} = \Pr_{X_i \sim D_i}\left[X_i \le v_{i, j}\right] (j\in[s_i])$.
%We define the \emph{CDF vector} of $D$ as $F^D = (F^D_{i, j})_{i\in [m], j\in[s_i]}$.
%The oracle takes a CDF vector $F^D$ (which describes $D$) as input, and can output a super arm $S' \in \calF$ such that $r_D(S') \ge \alpha \cdot \max_{S \in \calF} \{ r_D(S)\}$.

\section{SDCB Algorithm} \label{sec:alg}

%In this section, we consider the special case in which each outcome distribution $ D_i$ ($i\in[m]$) has a known finite support.
%Let $\supp(D_i) = \{v_{i, 1}, v_{i, 2}, \ldots, v_{i, s_i}\}$, and let $s$ be a finite upper bound on the support size of each $D_i$, i.e., $s_i = |\supp(D_i)| \le s$ ($\forall i\in[m]$).
%We assume $0\le v_{i, 1} <  v_{i, 2} <  \cdots < v_{i, s_i} \le 1$ for each $i\in[m]$.
%We do not require Assumption~\ref{assum:lipschitz} in this section. That assumption will be used for the general distribution case (in Section~\ref{sec:discretization}).

%\subsection{Algorithm}

\begin{algorithm}[t]
	\caption{\texttt{SDCB} (Stochastically dominant confidence bound)}  
	\label{alg:SDCB}
	\begin{algorithmic}[1]
		\STATE Throughout the algorithm, for each arm $i\in[m]$, maintain: (i) a counter $T_i$ which stores the number of times arm $i$ has been played so far, and
		 (ii) the empirical distribution $\hat D_i$ of the observed outcomes from arm $i$ so far, which is represented by its CDF $\hat F_i$% defined as $\hat F_i(x) = \frac{\text{number of outcomes from arm $i$ that are no larger than $x$}}{T_i}$ ($0\le x \le 1$).
		 %Let $\hat F_i$ be the CDF of $\hat D_i$.
		 \vspace{6pt}
		\STATE // Initialization \label{line:init-start}
		\FOR{$i=1$ \TO $m$} 
		\STATE // Action in the $i$-th round
		\STATE Play a super arm $S_i$ that contains arm $i$%, 		 observe the outcome $X^{(i)}_i$ from arm $i$, and		 find $k\in[s_i]$ such that $X^{(i)}_i = v_{i, k}$
		%\STATE $\hat F_{i, j} \leftarrow 1  \qquad \forall k\le j \le s_i$
		%\STATE $\hat F_{i, j} \leftarrow 0 \qquad \forall 1\le j \le k-1$
		%\STATE $T_i \leftarrow 1$
		\STATE Update $T_j$ and $\hat F_j$ for each $j\in S_i$
		\ENDFOR \label{line:init-end}
		\vspace{6pt}
		
		\FOR{$t=m+1, m+2, \ldots$} 
		\STATE // Action in the $t$-th round
		%\FOR{$i=1, 2, \ldots, m$}
		%\STATE $\wideubar{F}_{i, j} \leftarrow \max \{ \hat F_{i, j} - \sqrt{\frac{3\ln(\lambda t)} {2T_i}}, 0 \} \qquad \forall 1\le j \le s_i-1$ 
		%\STATE $\wideubar{F}_{i, s_i} \leftarrow 1$
		\STATE \label{line:sdcb}
			For each $i\in[m]$, let $\wideubar D_i$ be a distribution whose CDF $\wideubar F_i$ is
		\[
		\wideubar{F}_{i}(x) = \begin{cases}
		\max \{ \hat F_{i}(x) - \sqrt{\frac{3\ln t} {2T_i}}, 0\}, & 0\le x < 1 \\
		1, & x = 1
		\end{cases}
		\]
		%\ENDFOR
		\STATE \label{line:oracle}
			Play the super arm $S_t \leftarrow \mathsf{Oracle}(\wideubar D)$, where $\wideubar D = \wideubar D_1 \times \wideubar D_2 \times \cdots \times \wideubar D_m$ %$\wideubar F = (\wideubar F_{i, j})_{i\in[m], j\in[s_i]}$
		\STATE Update $T_j$ and $\hat F_j$ for each $j\in S_t$
		%\STATE Play the super arm $S_t$
%		\FORALL{$i\in S_t$}
%		\STATE Observe the outcome $X^{(t)}_i$ from arm $i$,
%		and find $k\in[s_i]$ such that $X^{(t)}_i = v_{i, k}$
%		\STATE $\hat F_{i, j} \leftarrow \frac{ T_i \cdot \hat F_{i, j}+1  } {T_i + 1} \qquad \forall k \le j \le s_i$
%		\STATE $\hat F_{i, j} \leftarrow \frac{ T_i \cdot \hat F_{i, j} }  {T_i + 1} \qquad \forall 1 \le j \le k-1$
%		\STATE $T_i \leftarrow T_i + 1$
%		\ENDFOR
		\ENDFOR
	\end{algorithmic}
\end{algorithm}

We present our algorithm \emph{stochastically dominant confidence bound (SDCB)} in Algorithm \ref{alg:SDCB}.
Throughout the algorithm, we store, in a variable $T_i$, the number of times the outcomes from arm $i$ are observed so far.
We also maintain the empirical distribution $\hat D_i$ of the observed outcomes from arm $i$ so far, which can be represented by its CDF $\hat F_i$: for $x\in[0, 1]$, the value of $\hat F_i(x)$ is just the fraction of the observed outcomes from arm $i$ that are no larger than $x$.
Note that $\hat F_i$ is always a step function which has ``jumps'' at the points that are observed outcomes from arm $i$.
Therefore it suffices to store these discrete points as well as the values of $\hat F_i$ at these points in order to store the whole function $\hat F_i$.
Similarly, the later computation of stochastically dominant CDF $\wideubar F_i$ (line~\ref{line:sdcb}) only requires computation
	at these points, and the input to the offline oracle only needs to
	provide these points and corresponding CDF values (line~\ref{line:oracle}).

The algorithm starts with $m$ initialization rounds in which each arm is played at least once\footnote{Without loss of generality, we assume that each arm $i\in[m]$ is contained in at least one super arm.} (lines~\ref{line:init-start}-\ref{line:init-end}).
In the $t$-th round ($t> m$), the algorithm consists of three steps.
First, it calculates for each $i\in[m]$ a distribution $\wideubar D_i$ whose CDF $\wideubar F_i$ is obtained by lowering the CDF $\hat F_i$ (line~\ref{line:sdcb}).
%a lower confidence bound $\wideubar F_{i, j}$ on each CDF value $F^D_{i, j}$ (lines 11-14).
%We ensure $0\le \wideubar F_{i, 1} \le \wideubar F_{i, 2} \le \cdots \le \wideubar  F_{i, s_i}=1 (\forall i\in[m])$, so $\wideubar F = (\wideubar F_{i, j})_{i\in[m], j\in[s_i]}$ is a valid CDF vector of some distribution $\wideubar D = \wideubar D_1 \times \cdots \times \wideubar D_m$.
The second step is to call the $\alpha$-approximation oracle with the newly constructed distribution $\wideubar D = \wideubar D_1 \times \cdots \times \wideubar D_m$ as input (line~\ref{line:oracle}), and thus the super arm $S_t$ output by the oracle satisfies
$r_{\wideubar D}(S_t) \ge \alpha \cdot \max_{S \in \calF} \{ r_{\wideubar D}(S) \}$.
Finally, the algorithm chooses the super arm $S_t$ to play, observes the outcomes from all arms in $S_t$, and updates $T_j$'s and $\hat F_{j}$'s accordingly for each $j\in S_t$.

The idea behind our algorithm is the \emph{optimism in the face of uncertainty} principle, which is the key principle behind UCB-type algorithms.
Our algorithm ensures that with high probability we have $\wideubar F_i(x) \le F_i(x)$ simultaneously for all $i\in[m]$ and all $x\in[0, 1]$, where $F_i$ is the CDF of the outcome distribution $D_i$. This means that each $\wideubar D_i$ has \emph{first-order stochastic dominance} over $D_i$.\footnote{We remark that while $\wideubar F_{i}(x)$ is a numerical lower confidence bound on $F_{i}(x)$ for all $x\in[0,1]$, at the distribution level, $\wideubar D_i$ serves as a ``stochastically dominant (upper) confidence bound'' on $D_i$.}
Then from the monotonicity property of $R(x, S)$ (Assumption~\ref{assum:monotone}) we know that $r_{\wideubar D}(S) \ge r_D(S)$ holds for all $S\in\calF$ with high probability.
Therefore $\wideubar D$ provides an ``optimistic'' estimation on the expected reward from each super arm.

%Note that we use a parameter $\lambda$ to control the confidence radius (line 12).
%In this section we can just set $\lambda$ to be a constant (e.g. $\lambda=1$).
%For the general distribution case in Section~\ref{sec:discretization}, it will be chosen more carefully to work with discretization.

%\subsection{Regret Bounds}
\paragraph{Regret Bounds.}
We prove $O(\log T)$ distribution-dependent and $O (\sqrt{T\log T})$ distribution-independent upper bounds on the regret of \texttt{SDCB} (Algorithm~\ref{alg:SDCB}).

We call a super arm $S$ \emph{bad} if $r_D(S) < \alpha \cdot r_D(S^*)$.
For each super arm $S$, we define
\[
\Delta_S = \max\{ \alpha \cdot r_D(S^*) - r_D(S), 0 \}.
\]
Let $\calF_\bad = \{S\in\calF \mid \Delta_S > 0 \}$,
which is the set of all \emph{bad} super arms. 
Let $E_\bad \subseteq[m]$ be the set of arms that are contained in at least one \emph{bad} super arm.
%\[
%E_\bad = \{i\in[m] \mid \text{there exists } S \in \calF_\bad \text{ such that } i\in S \}.
%\]
For each $i \in E_\bad$, we define
\begin{equation*}
	\Delta_{i, \min} = %\alpha \cdot r_D(S^*) - \max\{ r_D(S) \mid S \in \calF_\bad, i\in S\}.
	\min\{ \Delta_S \mid S \in \calF_\bad, i\in S \}.
\end{equation*}
Recall that $M$ is an upper bound on the reward value (Assumption~\ref{assum:bounded}) and $K = \max_{S \in \calF} |S|$.

%\begin{theorem} \label{thm:regret-bound-discrete-main}
%	Suppose each $D_i$ has a known finite support of size at most $s$.
%	Choosing $\lambda = 1$, a distribution-dependent upper bound on the $\alpha$-approximation regret of \texttt{SDCB} in $T$ rounds is
%	\begin{equation*}
%		M^2 K  \sum_{i\in E_\bad} \frac{2136}{\Delta_{i, \min}}  \ln T  + 4 \alpha Msm,
%	\end{equation*}
%	and a distribution-independent upper bound is
%	\begin{equation*}
%		93 M  \sqrt{mK T\ln T} + 4 \alpha Msm.
%	\end{equation*}
%\end{theorem}

\begin{theorem} \label{thm:regret-bound-discrete}
	%Suppose the outcome distribution $D_i$ of each arm $i\in[m]$ has a known finite support of size at most $s$.
	A distribution-dependent upper bound on the $\alpha$-approximation regret of \texttt{SDCB} (Algorithm~\ref{alg:SDCB}) in $T$ rounds is
	\begin{equation*}
	M^2 K  \sum_{i\in E_\bad} \frac{2136}{\Delta_{i, \min}}  \ln T  + \left( \frac{\pi^2}{3}  + 1 \right) \alpha Mm,
	\end{equation*}
	and a distribution-independent upper bound is
	\begin{equation*}
	93 M  \sqrt{mK T \ln T} + \left( \frac{\pi^2}{3} + 1 \right) \alpha Mm.
	\end{equation*}
\end{theorem}

The proof of Theorem~\ref{thm:regret-bound-discrete} is given in Appendix~\ref{subsec:proof-regret-bound-discrete}.
The main idea is to reduce our analysis on general reward functions satisfying Assumptions~\ref{assum:indep}-\ref{assum:monotone} 
to the one in~\cite{Kveton15} that deals with the \emph{summation} reward function $R(x, S) = \sum_{i\in S} x_i$.
Our analysis relies on the Dvoretzky-Kiefer-Wolfowitz inequality~\cite{dvoretzky1956asymptotic,massart1990tight}, which gives a uniform concentration bound
	on the empirical CDF of a distribution.

%We remark that Theorem~\ref{thm:regret-bound-discrete} holds even if the distribution $D$ can take values on $\R^m$ instead of just $[0,1]^m$.
%We use $[0,1]^m$ for our main model since our later results require bounded support.

\paragraph{Applying Our Algorithm to the Previous CMAB Framework.}

Although our focus is on general reward functions, we note that when \texttt{SDCB} is applied to the previous CMAB framework where
the expected reward depends only on the means of the random variables,
it can achieve the same regret bounds as the previous \emph{combinatorial upper confidence bound (\texttt{CUCB})} algorithm in~\cite{chen2016cmab, Kveton15}.

%We now briefly describe \texttt{CUCB}.
%Let
%$\mu_i = \E_{X\sim D}[X_i]$ be the mean outcome from arm $i$.
%\texttt{CUCB} maintains (for each $i$) the empirical mean $\hat \mu_i$ of the previous outcomes from arm $i$, and in round $t$, it calculates an upper confidence bound of $\mu_i$ as $\bar \mu_i = \hat \mu_i + \sqrt{\frac{3\ln t}{2T_i}}$, where $T_i$ stores the number of observations from arm $i$ so far.
%Then it calls an oracle to select the super arm, with $\bar \mu_i$'s as input.

Let
$\mu_i = \E_{X \sim D}[X_i]$ be arm $i$'s mean outcome.
In each round \texttt{CUCB} calculates (for each arm $i$) an upper confidence bound $\bar \mu_i$ on $\mu_i$, 
	with the essential property that $\mu_i \le \bar \mu_i \le \mu_i + \Lambda_i$ holds with high probability, for some $\Lambda_i>0$.
In \texttt{SDCB}, we use $\wideubar D_i$ as a stochastically dominant confidence bound of $D_i$.
We can show that $\mu_i \le \E_{Y_i \sim \wideubar D_i} [Y_i] \le \mu_i + \Lambda_i$ holds with high probability, with the same interval length $\Lambda_i$ as in \texttt{CUCB}. (The proof is given in Appendix~\ref{subsec:previous-cmab}.)
Hence, the analysis in \cite{chen2016cmab, Kveton15} can be applied to \texttt{SDCB}, resulting in
	the same regret bounds.% for the term containing $T$ (the term not containing $T$ will have an extra factor of $s$
	%due to a union bound for all supported points).
We further remark that in this case we \emph{do not} need the three assumptions stated in Section~\ref{sec:problem}
	(in particular the independence assumption on $X_i$'s):
	the summation reward case just works as in~\cite{Kveton15} and the nonlinear reward case relies on the properties of
	monotonicity and bounded smoothness used in~\cite{chen2016cmab}.

\section{Improved SDCB Algorithm by Discretization} \label{sec:discretization}

In Section~\ref{sec:alg}, we have shown that our algorithm \texttt{SDCB} achieves near-optimal regret bounds. However, that algorithm might suffer from large running time and memory usage.
Note that, in the $t$-th round, an arm $i$ might have been observed $t-1$ times already, and it is possible that all the observed values from arm $i$ are different (e.g., when arm $i$'s outcome distribution $D_i$ is continuous). In such case, it takes $\Theta(t)$ space to store the empirical CDF $\hat F_i$ of the observed outcomes from arm $i$, and both calculating the stochastically dominant CDF $\wideubar F_i$ and updating $\hat F_i$ take $\Theta(t)$ time. Therefore, the worst-case space usage of \texttt{SDCB} in $T$ rounds is $\Theta(T)$, and the worst-case running time is $\Theta(T^2)$ (ignoring the dependence on $m$ and $K$); here we do not count the time and space used by the offline computation oracle.

%in $T$ rounds \texttt{SDCB} uses $\Theta(T)$ space in the worst case, and the 
% running time in round $t$, not including the running time of the oracle, is $O(t)$ (ignoring the dependence on $m$ and $K$), and the total running time in $T$ rounds is $O(T^2)$.

In this section, we propose an improved algorithm \texttt{Lazy-SDCB} which reduces the worst-case memory usage and running time to $O(\sqrt T)$ and $O(T^{3/2})$, respectively, while preserving the $O(\sqrt{T \log T})$ distribution-independent regret bound.
To this end, we need an additional assumption on the reward function:
\begin{assumption}[Lipschitz-continuous reward function] \label{assum:lipschitz}
	There exists $C>0$ such that for any $S \in \calF$ and any $x, x'\in [0, 1]^m$, we have $\left| R(x, S) - R(x', S) \right| \le C\|x_S - x_S'\|_1$,
	where $\| x_S - x_S'\|_1 = \sum_{i\in S} |x_i - x_i'|$.
\end{assumption}

%In this section, we consider the general case where each $D_i$ is an arbitrary distribution over $[0, 1]$.
%Our algorithm performs a discretization step on the distributions and then applies \texttt{SDCB} (Algorithm~\ref{alg:SDCB}).
%We propose two algorithms based on the combination of a discretization step and \texttt{SDCB} (Algorithm~\ref{alg:SDCB}), one of which knows the time horizon $T$ in advance and the other does not. 
%Both algorithms achieve %distribution-independent regret of $ O(\sqrt {T \ln T})$.
%$O(\log T)$ distribution-dependent regret\footnote{The $O(\log T)$ distribution-dependent regret bound only holds for suficiently large $T$.} and $O (\sqrt{T\ln T})$ distribution-independent regret.
%	Note that we will make use of the Lipschitz-continuity condition given as Assumption~\ref{assum:lipschitz} which was not used in the finite support case in Section~\ref{sec:alg}.

%\subsection{Algorithms}

We first describe the algorithm when the time horizon $T$ is known in advance. %, which we call \texttt{SDCB-GDT} (SDCB for general distributions with known time horizon).
The algorithm is summarized in Algorithm~\ref{alg:Lazy-SDCB}.
We perform a \emph{discretization} on the distribution $D = D_1 \times \cdots \times D_m$ to obtain a discrete distribution $\tilde{D} = \tilde D_1 \times \cdots \times \tilde D_m$ such that (i) for $\tilde X \sim \tilde D$, $\tilde X_1, \ldots, \tilde X_m$ are also mutually independent, and (ii) every $\tilde D_i$ is supported on a set of equally-spaced values $\{ \frac1s, \frac2s, \ldots, 1\}$, where $s$ is set to be $\lceil \sqrt T \rceil$.
Specifically, we partition $[0, 1]$ into $s$ intervals: $I_1 = [0, \frac1s], I_2 = (\frac1s, \frac2s], \ldots, I_{s-1} = (\frac{s-2}{s}, \frac{s-1}{s}],I_s = (\frac{s-1}{s}, 1] $,
and define $\tilde{D_i}$ as
\begin{equation*} %\label{eqn:discretize}
	\Pr_{\tilde{X}_i \sim \tilde{D}_i} [\tilde{X}_i =  j/s ] = \Pr_{X_i \sim D_i}\left[  X_i \in I_j \right], \qquad j = 1, \ldots, s.
\end{equation*}
%Since each $\tilde D_i$ has a known finite support, \texttt{SDCB} (Algorithm~\ref{alg:SDCB}) can be applied to the discretized CMAB problem $([m], \calF, \tilde D, R)$.
For the CMAB problem $([m], \calF, D, R)$, our algorithm ``pretends'' that the outcomes are drawn from $\tilde D$ instead of $D$, by replacing any outcome $x\in I_j$ by $\frac js$ ($\forall j\in[s]$),
and then applies \texttt{SDCB} to the problem $([m], \calF, \tilde D, R)$.
Since each $\tilde D_i$ has a known support $\{ \frac1s, \frac2s, \ldots, 1\}$, the algorithm only needs to maintain the number of occurrences of each support value in order to obtain the empirical CDF of all the observed outcomes from arm $i$. Therefore, all the operations in a round can be done using $O(s) = O(\sqrt T)$ time and space, and the total time and space used by \texttt{Lazy-SDCB} are $O(T^{3/2})$ and $O(\sqrt T)$, respectively.

\begin{algorithm}[t]
	\caption{\texttt{Lazy-SDCB} with known time horizon}
	\label{alg:Lazy-SDCB}
	\begin{algorithmic}[1]
		\REQUIRE  time horizon $T$
		%\vspace{3pt}
		\STATE $s  \leftarrow \lceil  \sqrt T \rceil$ %\qquad ($\eta \ge 0$ is a parameter) %\bigg\lceil \sqrt{ \frac{6}{\pi^2} \max\left\{\frac{CK}{M}, 1\right\} \frac{T}{m} }\bigg\rceil$
		\STATE $I_j \leftarrow \begin{cases}
		[0, \frac 1s], & j = 1\\
		(\frac{j-1}{s}, \frac{j}{s}], & j = 2, \ldots, s
		\end{cases}$
		\STATE Invoke \texttt{SDCB} (Algorithm~\ref{alg:SDCB}) %with $\supp(\tilde D_i) = \{\frac1s, \frac2s, \ldots, 1\}$ ($\forall i\in[m]$) and $\lambda = (s-1)^{1/3}$ 
		for $T$ rounds, with the following change:
		whenever observing an outcome $x $ (from any arm), find $j\in[s]$ such that $x\in I_j$, and regard this outcome as $\frac js$
	\end{algorithmic}
\end{algorithm}

The discretization parameter $s$ in Algorithm~\ref{alg:Lazy-SDCB} depends on the time horizon $T$, which is why Algorithm~\ref{alg:Lazy-SDCB} has to know $T$ in advance.
%Nevertheless, we can overcome this shortcoming by varying the discretization parameter in the algorithm.
We can use the doubling trick to avoid the dependency on $T$.
We present such an algorithm (without knowing $T$) in Algorithm~\ref{alg:Lazy-SDCB-unknown-T}.
It is easy to see that Algorithm~\ref{alg:Lazy-SDCB-unknown-T} has the same asymptotic time and space usages as Algorithm~\ref{alg:Lazy-SDCB}.

\begin{algorithm}[t]
	\caption{\texttt{Lazy-SDCB} without knowing the time horizon}
	\label{alg:Lazy-SDCB-unknown-T}
	\begin{algorithmic}[1]
		%\REQUIRE $m$, $\calF$
		%\vspace{3pt}
		\STATE %$q \leftarrow \left\lceil \log_2 \left( \frac{\pi^2}{6} \cdot \frac{ Mm}{CK} \right) \right\rceil$
		$q \leftarrow \left\lceil \log_2 m \right\rceil$
		\STATE In rounds $1, 2, \ldots, 2^{q}$, invoke Algorithm~\ref{alg:Lazy-SDCB} with input $T = 2^q$
		\FOR{$k = q, q+1, q+2, \ldots$}
		\STATE In rounds $2^{k}+1, 2^{k}+2, \ldots, 2^{k+1}$, invoke Algorithm~\ref{alg:Lazy-SDCB} with input $T = 2^{k}$
		\ENDFOR
	\end{algorithmic}
\end{algorithm}

\paragraph{Regret Bounds.}
We show that both Algorithm~\ref{alg:Lazy-SDCB} and Algorithm~\ref{alg:Lazy-SDCB-unknown-T} achieve $O(\sqrt{T \log T})$ distribution-independent regret bounds.
% and that \texttt{SDCB-GDT} also achieves an $O(\log T)$ distribution-dependent regret bound for sufficiently large $T$.
The full proofs are given in Appendix~\ref{appdx:proof-general-case}.
Recall that $C$ is the coefficient in the Lipschitz condition in Assumption~\ref{assum:lipschitz}.

\begin{theorem} \label{thm:bound-know-T}
	Suppose the time horizon $T$ is known in advance.
	Then the $\alpha$-approximation regret of Algorithm~\ref{alg:Lazy-SDCB} in $T$ rounds is at most
	%Then a distribution-independent upper bound on the $\alpha$-approximation regret of \texttt{SDCB-GDT} (with parameter $\eta \ge 0$) in $T$ rounds is
	\[
	  93 M \sqrt{  mK T \ln T} + 2CK\sqrt T + \left( \frac{\pi^2}{3} + 1 \right)  \alpha Mm.
	\]
%	Moreover, if we have $T \ge \left( \frac{(2+\delta) CK}{\Delta_{i, \min}} \right)^{\frac{1}{1+\eta}}$ for all $i\in E_\bad$ ($\delta>0$), then a distribution-dependent upper bound on the $\alpha$-approximation regret of \texttt{SDCB-GDT} (with parameter $\eta \ge 0$) in $T$ rounds is
%	\[
%	M^2K  \sum_{i \in E_\bad} \frac{712 (4+\eta) (1+2/\delta)}{ \Delta_{i, \min}} \ln  T + \left( \frac{\pi^2}{3} + 1 \right)  \alpha Mm + \frac{4CK}{T^\eta}.
%	\]
\end{theorem}

\begin{proof}[Proof Sketch]
	%The full proof is given in the supplementary material. Here we provide a sketch.
	The regret consists of two parts: (i) the regret for the discretized CMAB problem $([m], \calF, \tilde{D}, R)$, and (ii) the error due to discretization.
	We directly apply Theorem~\ref{thm:regret-bound-discrete} for the first part.
	For the second part, a key step is to show $\left| r_D(S) - r_{\tilde{D}}(S) \right| \le CK/s$ for all $S \in \calF$ (see Appendix~\ref{subsec:disc-error}).
\end{proof}

%
%\begin{lemma} \label{lem:disc-error}
%	For any $S \in \calF$, we have
%	\[
%	\left| r_D(S) - r_{\tilde{D}}(S) \right| \le CK/s.
%	\]
%\end{lemma}

%The proof of Lemma~\ref{lem:disc-error} is given in the supplementary material.
%postponed to Appendix~\ref{appdx:proof-general-case}.

%The regret bound for \texttt{SDCB-GD} is a corollary of Theorem~\ref{thm:bound-know-T}.

\begin{theorem} \label{thm:bound-not-know-T}
	For any time horizon $T \ge 2$, the $\alpha$-approximation regret of Algorithm~\ref{alg:Lazy-SDCB-unknown-T} in $T$ rounds is at most
	\[
	318M \sqrt{mKT  \ln T} + 7 CK \sqrt{T} + 10 \alpha Mm \ln T.
	\]
\end{theorem}

\section{Applications}

We describe the $K$-MAX problem and the class of expected utility maximization problems
	as  applications of our general CMAB framework.

\paragraph{The $K$-MAX Problem.}
In this problem, the player is allowed to select at most $K$ arms from the set of $m$ arms in each round, and the reward is the maximum one among the outcomes from the selected arms. In other words, the set of feasible super arms is $\calF = \left\{S\subseteq [m] \,\big|\, |S| \le K \right\}$, and the reward function is $R(x, S) = \max_{i\in S} x_i$.
%\begin{equation*} %\label{eqn:max-reward-function}
%	R(x, S) = \max_{i\in S} x_i.
%\end{equation*}
It is easy to verify that this reward function satisfies Assumptions~\ref{assum:bounded}, \ref{assum:monotone} and~\ref{assum:lipschitz} with $M=C=1$.

%This reward function is clearly bounded in $[0, 1]$, and satisfies the monotonicity requirement. We can also easily verify its Lipschitz continuity: for any $S\in \calF$ and any $x, x' \in [0, 1]^m$, we have $|\max_{i\in S} x_i - \max_{i\in S} x_i'| \le \|x_S - x_S'\|_1$.
%Therefore, we can set $M = C = 1$ for the $K$-MAX problem.

Now we consider the corresponding offline $K$-MAX problem of selecting at most $K$ arms from $m$ independent arms, with the largest expected reward.
It can be implied by a result in \cite{goel2010probe} that finding the exact optimal solution is NP-hard, so we resort to approximation algorithms.
We can show, using submodularity, that a simple greedy algorithm can achieve a $(1-1/e)$-approximation.
Furthermore, we give the first PTAS for this problem. 
Our PTAS can be generalized to constraints other than the cardinality constraint $|S|\le K$, including $s$-$t$ simple paths, matchings, knapsacks, etc.
The algorithms and corresponding proofs are given in Appendix~\ref{appdx:offline-kmax}.

\begin{theorem}
	\label{thm:kmax}
	There exists a PTAS for the offline $K$-MAX problem.
	In other words, for any constant $\epsilon  >0$, there is a polynomial-time $(1-\epsilon)$-approximation algorithm for the offline $K$-MAX problem.
\end{theorem}

We thus can apply our \texttt{SDCB} algorithm to the $K$-MAX bandit problem and obtain $O(\log T)$ distribution-dependent and $\tilde{O}(\sqrt T)$ distribution-independent regret bounds according to Theorem~\ref{thm:regret-bound-discrete},
or can apply \texttt{Lazy-SDCB} to get $\tilde{O}(\sqrt T)$ distribution-independent bound according to Theorem~\ref{thm:bound-know-T} or~\ref{thm:bound-not-know-T}.

Streeter and Golovin \cite{Streeter2008} study an \emph{online submodular maximization} problem in the oblivious adversary model.
In particular, their result can cover the stochastic $K$-MAX bandit problem as a special case, and an $O(K\sqrt{mT\log m})$ upper bound on the $(1-1/e)$-regret can be shown.
While the techniques in~\cite{Streeter2008} can only give a bound on the $(1-1/e)$-approximation regret for $K$-MAX, we can obtain the first $\tilde{O} (\sqrt T)$ bound on the $(1-\epsilon)$-approximation regret for any constant $\epsilon>0$, using our PTAS as the offline oracle.
Even when we use the simple greedy algorithm as the oracle, our experiments show that \texttt{SDCB} performs significantly better than the algorithm in~\cite{Streeter2008} (see Appendix~\ref{appdx:experiment}).

%Finally, note that the $K$-MAX problem corresponds to the first-price auction in our illustrative
%example.
%If we replace first-price auction by second-price auction or other more complicated forms
%of auctions (e.g. Myerson optimal auction), in principle our \texttt{SDCB} algorithm still works,
%but it remains unclear whether efficient exact or approximation algorithms
%exist for the offline optimization problems in these auction formats. 
%In other words, the difficulty of these problems is on the offline optimization side, rather than
%the online learning side.

\paragraph{Expected Utility Maximization.}
Our framework can also be applied to  reward functions of the form
%Consider the problem where 
%the player is allowed to select a set $S$ of at most $K$ arms from the set of $m$ arms in each round.
%But now, the reward is the utility of the total reward of $S$, i.e.,
$R(x, S) = u(\sum_{i\in S}x_i)$, where $u(\cdot)$ is an increasing utility function.
%In other words, the set of feasible super arms is $\calF = \left\{S\subseteq [m] \,\big|\, |S| \le K \right\}$, and the reward function is $R(x, S) = u(\sum_{i\in S} x_i)$. We can of course consider more general feasibility constraints.
The corresponding offline problem is 
to maximize
the expected utility $\E[u(\sum_{i\in S}x_i)]$ subject to a feasibility constraint $S \in \calF$.
Note that if $u$ is nonlinear, the expected utility may not be a function of 
the means of the arms in $S$.
Following the celebrated von Neumann-Morgenstern expected utility theorem,
nonlinear utility functions have been extensively used to capture
risk-averse or risk-prone behaviors in economics (see e.g., \cite{Fishburn82}),
while linear utility functions correspond to risk-neutrality.

Li and Deshpande \cite{LiD11} obtain a PTAS for the expected utility maximization (EUM)
problem for several classes of utility functions
(including for example increasing concave
functions which typically indicate risk-averseness), 
and a large class
of feasibility constraints (including cardinality constraint, 
$s$-$t$ simple paths, matchings, and knapsacks).
Similar results for other utility functions and feasibility 
constraints can be found in \cite{yu2016maximizing,li2013stochastic,bhalgat2014utility}.
In the online problem, we can apply our algorithms, using their PTASs as
the offline oracle. 
Again, we can obtain the first tight regret bounds on the $(1-\epsilon)$-approximation regret for any $\epsilon>0$, 
for the class of online EUM problems.

%\section{Future Work}
%
%%In this paper, we extend the stochastic CMAB framework to allow general reward functions
%%whose expected value may not only depend on the means of its constituent random variables.
%%We propose the \texttt{SDCB} algorithm  and a discretization method
%%to achieve near-optimal regret bounds.
%%
%There are many opportunities to extend our work.
%First, one may investigate if some of the assumptions can be removed
%and in that case, if new algorithm is needed.
%Second, it is interesting to look into particular problem instances that can be solved by our framework
%to see if any improvement is possible.
%Finally, for some hard offline instances, such as the $K$-MAX variant for second-price auction, it would be interesting to see if offline solution exists, and even if only heuristics
%are available, it would be interesting to investigate whether the online learning part could help guarantee any regret performance when
%we use a heuristic offline oracle as the benchmark.

\subsubsection*{Acknowledgments}

Wei Chen was supported in part by the National Natural Science Foundation of China (Grant No. 61433014).
Jian Li and Yu Liu were supported in part by the National Basic Research Program of China grants 2015CB358700, 2011CBA00300, 2011CBA00301, and the National NSFC grants 61033001, 61361136003.
The authors would like to thank Tor Lattimore for referring to us the DKW inequality.

{
\small
\bibliographystyle{plain}
\bibliography{ref} 

\begin{thebibliography}{10}
\setlength{\itemsep}{1.5pt}
\bibitem{audibert2009minimax}
Jean-Yves Audibert and S{\'e}bastien Bubeck.
\newblock Minimax policies for adversarial and stochastic bandits.
\newblock In {\em COLT}, pages 217--226, 2009.

\bibitem{auer2002finite}
Peter Auer, Nicolo Cesa-Bianchi, and Paul Fischer.
\newblock Finite-time analysis of the multiarmed bandit problem.
\newblock {\em Machine learning}, 47(2-3):235--256, 2002.

\bibitem{auer2002nonstochastic}
Peter Auer, Nicolo Cesa-Bianchi, Yoav Freund, and Robert~E. Schapire.
\newblock The nonstochastic multiarmed bandit problem.
\newblock {\em SIAM Journal on Computing}, 32(1):48--77, 2002.

\bibitem{bhalgat2014utility}
Anand Bhalgat and Sanjeev Khanna.
\newblock A utility equivalence theorem for concave functions.
\newblock In {\em IPCO}, pages 126--137. Springer, 2014.

\bibitem{BubeckC12}
S{\'{e}}bastien Bubeck and Nicol{\`{o}} Cesa{-}Bianchi.
\newblock Regret analysis of stochastic and nonstochastic multi-armed bandit
  problems.
\newblock {\em Foundations and Trends in Machine Learning}, 5(1):1--122, 2012.

\bibitem{cesa2012combinatorial}
Nicolo Cesa-Bianchi and G{\'a}bor Lugosi.
\newblock Combinatorial bandits.
\newblock {\em Journal of Computer and System Sciences}, 78(5):1404--1422,
  2012.

\bibitem{chen2014pure}
Shouyuan Chen, Tian Lin, Irwin King, Michael~R. Lyu, and Wei Chen.
\newblock Combinatorial pure exploration of multi-armed bandits.
\newblock In {\em NIPS}, 2014.

\bibitem{chen2016cmab}
Wei Chen, Yajun Wang, Yang Yuan, and Qinshi Wang.
\newblock Combinatorial multi-armed bandit and its extension to
  probabilistically triggered arms.
\newblock {\em Journal of Machine Learning Research}, 17(50):1--33, 2016.

\bibitem{Combes2015}
Richard Combes, M.~Sadegh Talebi, Alexandre Proutiere, and Marc Lelarge.
\newblock Combinatorial bandits revisited.
\newblock In {\em NIPS}, 2015.

\bibitem{dvoretzky1956asymptotic}
Aryeh Dvoretzky, Jack Kiefer, and Jacob Wolfowitz.
\newblock Asymptotic minimax character of the sample distribution function and
  of the classical multinomial estimator.
\newblock {\em The Annals of Mathematical Statistics}, pages 642--669, 1956.

\bibitem{Fishburn82}
P.~C. Fishburn.
\newblock {\em The foundations of expected utility}.
\newblock Dordrecht: Reidel, 1982.

\bibitem{Yi2012}
Yi~Gai, Bhaskar Krishnamachari, and Rahul Jain.
\newblock Combinatorial network optimization with unknown variables:
  Multi-armed bandits with linear rewards and individual observations.
\newblock {\em IEEE/ACM Transactions on Networking}, 20(5):1466--1478, 2012.

\bibitem{goel2006asking}
Ashish Goel, Sudipto Guha, and Kamesh Munagala.
\newblock Asking the right questions: Model-driven optimization using probes.
\newblock In {\em PODS}, pages 203--212. ACM, 2006.

\bibitem{goel2010probe}
Ashish Goel, Sudipto Guha, and Kamesh Munagala.
\newblock How to probe for an extreme value.
\newblock {\em ACM Transactions on Algorithms (TALG)}, 7(1):12:1--12:20, 2010.

\bibitem{GMM14}
Aditya Gopalan, Shie Mannor, and Yishay mansour.
\newblock Thompson sampling for complex online problems.
\newblock In {\em ICML}, pages 100--108, 2014.

\bibitem{KWAEE14}
Branislav Kveton, Zheng Wen, Azin Ashkan, Hoda Eydgahi, and Brian Eriksson.
\newblock Matroid bandits: Fast combinatorial optimization with learning.
\newblock In {\em UAI}, pages 420--429, 2014.

\bibitem{Kveton15cascade}
Branislav Kveton, Zheng Wen, Azin Ashkan, and Csaba Szepesv\'{a}ri.
\newblock Combinatorial cascading bandits.
\newblock In {\em NIPS}, 2015.

\bibitem{Kveton15}
Branislav Kveton, Zheng Wen, Azin Ashkan, and Csaba Szepesv\'{a}ri.
\newblock Tight regret bounds for stochastic combinatorial semi-bandits.
\newblock In {\em AISTATS}, pages 535--543, 2015.

\bibitem{lai1985asymptotically}
Tze~Leung Lai and Herbert Robbins.
\newblock Asymptotically efficient adaptive allocation rules.
\newblock {\em Advances in applied mathematics}, 6(1):4--22, 1985.

\bibitem{LiD11}
Jian Li and Amol Deshpande.
\newblock Maximizing expected utility for stochastic combinatorial optimization
  problems.
\newblock In {\em FOCS}, pages 797--806, 2011.

\bibitem{li2013stochastic}
Jian Li and Wen Yuan.
\newblock Stochastic combinatorial optimization via poisson approximation.
\newblock In {\em STOC}, pages 971--980, 2013.

\bibitem{lin2014combinatorial}
Tian Lin, Bruno Abrahao, Robert Kleinberg, John Lui, and Wei Chen.
\newblock Combinatorial partial monitoring game with linear feedback and its
  applications.
\newblock In {\em ICML}, pages 901--909, 2014.

\bibitem{lin2015online}
Tian Lin, Jian Li, and Wei Chen.
\newblock Stochastic online greedy learning with semi-bandit feedbacks.
\newblock In {\em NIPS}, 2015.

\bibitem{massart1990tight}
Pascal Massart.
\newblock The tight constant in the dvoretzky-kiefer-wolfowitz inequality.
\newblock {\em The Annals of Probability}, pages 1269--1283, 1990.

\bibitem{nemhauser1978submod}
George~L. Nemhauser, Laurence~A. Wolsey, and Marshall~L. Fisher.
\newblock An analysis of approximations for maximizing submodular set functions
  -- \uppercase\expandafter{\romannumeral1}.
\newblock {\em Mathematical Programming}, 14(1):265--294, 1978.

\bibitem{Streeter2008}
Matthew Streeter and Daniel Golovin.
\newblock An online algorithm for maximizing submodular functions.
\newblock In {\em NIPS}, 2008.

\bibitem{yu2016maximizing}
Jiajin Yu and Shabbir Ahmed.
\newblock Maximizing expected utility over a knapsack constraint.
\newblock {\em Operations Research Letters}, 44(2):180--185, 2016.

\end{thebibliography}
}
 %\ifx \allfiles \undefined
%\input{def}
%\fi

%\clearpage

\appendix

%\setcounter{assumption}{4}
%\setcounter{theorem}{4}
%\setcounter{algorithm}{3}

%\section*{Supplementary Material for ``Combinatorial Multi-Armed Bandit with General Reward Functions''}

\section*{Appendix}

\section{Missing Proofs from Section~\ref{sec:alg}} \label{appdx:proof-discrete-case}

\subsection{Proof of Theorem~\ref{thm:regret-bound-discrete}} \label{subsec:proof-regret-bound-discrete}

We present the proof of Theorem~\ref{thm:regret-bound-discrete} in four steps.
In Section~\ref{subsubsec:l1-dist}, we review the $L_1$ distance between two distributions and present a property of it.
In Section~\ref{subsubsec:dkw}, we review the Dvoretzky-Kiefer-Wolfowitz (DKW) inequality, which is a strong concentration result for empirical CDFs.
In Section~\ref{subsubsec:lemmas-discrete-case}, we prove some key technical lemmas.
Then we complete the proof of Theorem~\ref{thm:regret-bound-discrete} in Section~\ref{subsubsec:proof-thm-1}.

\subsubsection{The $L_1$ Distance between Two Probability Distributions} \label{subsubsec:l1-dist}

For simplicity, we only consider discrete distributions with finite supports -- this will be enough for our purpose.

Let $P$ be a probability distribution.
%For all $x\in \supp(P)$, let $P(x) = \Pr_{X\sim P}[X=x]$; for any $x\notin \supp(P)$, let $P(x)=0$.
For any $x$, let $P(x) = \Pr_{X\sim P}[X=x]$.
We write $P= P_1 \times P_2 \times \cdots \times P_n$ if the (multivariate) random variable $X\sim P$ can be written as $X = (X_1, X_2, \ldots, X_n)$, where $X_1, \ldots, X_n$ are mutually independent and $X_i \sim P_i$ ($\forall i\in[n]$).

For two distributions $P$ and $Q$, their \emph{$L_1$ distance} is defined as
\[
L_1(P, Q) = \sum_{x} |P(x) - Q(x)|,
\]
where the summation is taken over $x\in \supp(P) \cup \supp(Q)$.

The $L_1$ distance has the following property. It is a folklore result and we provide a proof for completeness.

\begin{lemma} \label{lem:l1-product-distribuction}
	Let $P= P_1 \times P_2 \times \cdots \times P_n$ and $Q= Q_1 \times Q_2 \times \cdots \times Q_n$ be two probability distributions. Then we have
	\begin{equation} \label{eqn:l1-ineq}
	L_1(P, Q) \le \sum_{i=1}^n L_1(P_i, Q_i).
	\end{equation}
\end{lemma}
\begin{proof}
	We prove \eqref{eqn:l1-ineq} by induction on $n$.
	
	When $n=2$, we have
	\begin{align*}
	L_1(P, Q) &= \sum_{x} \sum_{y} |P(x, y) - Q(x, y)|\\
	&=  \sum_{x} \sum_{y} |P_1(x) P_2(y) - Q_1(x)Q_2(y)|\\
	&\le \sum_{x} \sum_{y} \left( |P_1(x) P_2(y) - P_1(x)Q_2(y)| + |P_1(x) Q_2(y) - Q_1(x)Q_2(y)| \right) \\
	&= \sum_x P_1(x) \sum_y |P_2(y) - Q_2(y)| + \sum_y Q_2(y) \sum_x |P_1(x) - Q_1(x)| \\
	&= 1 \cdot L_1(P_2, Q_2) + 1\cdot L_1(P_1, Q_1) \\
	&= \sum_{i=1}^2 L_1(P_i, Q_i).
	\end{align*}
	Here the summation is taken over $x \in \supp(P_1) \cup \supp(Q_1)$ and $y \in \supp(P_2) \cup \supp(Q_2)$.
	
	Suppose \eqref{eqn:l1-ineq} is proved for $n=k-1$ ($k\ge3$). When $n=k$, using the results for $n=k-1$ and $n=2$, we get
	\begin{align*}
	L_1(P, Q) &\le \sum_{i=1}^{k-2} L_1(P_i, Q_i) + L_1(P_{k-1}\times P_k, Q_{k-1}\times Q_k) \\
	&\le \sum_{i=1}^{k-2} L_1(P_i, Q_i) + L_1(P_{k-1}, Q_{k-1}) + L_1(P_k, Q_k) \\
	&= \sum_{i=1}^k L_1(P_i, Q_i).
	\end{align*}
	This completes the proof.
\end{proof}

\subsubsection{The DKW Inequality} \label{subsubsec:dkw}

Consider a distribution $D$ with CDF $F(x)$. Let $\hat F_n(x)$ be the empirical CDF of $n$ i.i.d. samples $X_1, \ldots, X_n$ drawn from $D$, i.e., $\hat F_n(x) = \frac1n \sum_{i=1}^n \mathds{1}\{X_i\le x\}$ ($x \in \R$).\footnote{We use $\mathds{1}\{\cdot \}$ to denote
the indicator function, i.e., $\mathds{1}\{\calH\} = 1$ if an event $\calH$ happens,
and $\mathds{1}\{\calH\} = 0$ if it does not happen.}
Then we have:
\begin{lemma}[Dvoretzky-Kiefer-Wolfowitz inequality~\cite{dvoretzky1956asymptotic,massart1990tight}] \label{lem:dkw}
	For any $\epsilon>0$ and any $n \in \Z_+$, we have
	\[
	\Pr\left[ \sup_{x\in\R} \left| \hat F_n(x) - F(x) \right| \ge \epsilon \right] \le 2 e^{-2n \epsilon^2}.
	\]
\end{lemma}

Note that for any fixed $x\in \R$, from the Chernoff bound we have $\Pr\left[  \left| \hat F_n(x) - F(x) \right| \ge \epsilon \right] \le 2 e^{-2n \epsilon^2}.$
The DKW inequality states a stronger guarantee that the Chernoff concentration holds simultaneously for all $x\in\R$.

\subsubsection{Technical Lemmas} \label{subsubsec:lemmas-discrete-case}

%Recall that we have $\supp(D_i) = \{v_{i, 1}, v_{i, 2}, \ldots, v_{i, s_i}\}$ ($0\le v_{i, 1} < \cdots< v_{i, s_i}\le 1$) and $s_i = |\supp (D_i)| \le s$ ($\forall i\in[m]$).
%Let
%\[
%\Sigma = \left\{ P = P_1 \times P_2 \times \cdots \times P_m \, \big| \, \supp(P_i) \subseteq \{v_{i, 1}, \ldots, v_{i, s_i}\}, i=1, 2, \ldots, m \right\},
%\]
%which is the set of all possible outcome distributions.
%Recall that for any distribution $P\in\Sigma$, we define its CDF vector as $F^P = (F^P_{i, j})_{i\in[m], j\in[s_i]}$, where $F^P_{i, j} = \Pr_{X_i \sim P_i} [X_i \le v_{i, j}]$.

The following lemma describes some properties of the expected reward $r_P(S) = \E_{X\sim P} [R(X, S)]$.

\begin{lemma} \label{lem:exp-reward-property}
	Let $P = P_1\times \cdots \times P_m$ and $P' = P_1'\times \cdots \times P_m'$ be two probability distributions over $[0, 1]^m$.
	Let $F_i$ and $F_i'$ be the CDFs of $P_i$ and $P_i'$, respectively ($i=1, \ldots, m$).
	Suppose each $P_i$ ($i\in[m]$) is a discrete distribution with finite support.
	\begin{enumerate}[(i)]
		\item If for any $i\in[m], x\in[0, 1]$ we have $F_i'(x) \le F_i(x)$,
		%$i\in[m], j\in[s_i]$ we have $F^{P'}_{i, j} \le F^{P}_{i, j}$,
		then for any super arm $S \in \calF$, we have
		\begin{equation*}
		r_{P'}(S) \ge r_P(S).
		\end{equation*}
		\item If for any $i\in[m], x\in[0, 1]$ we have $ F_i(x) - F_i'(x)  \le \Lambda_i$ ($\Lambda_i>0$), then for any super arm $S \in \calF$, we have
		\begin{equation*}
		r_{P'}(S) - r_P(S) \le 2M \sum_{i\in S} \Lambda_i.
		\end{equation*}
	\end{enumerate}
\end{lemma}
\begin{proof}
	It is easy to see why (i) is true. If we have $F_i'(x) \le F_i(x)$ for all $i\in[m]$ and $x\in[0, 1]$, then for all $i$, $P_i'$ has first-order stochastic dominance over $P_i$.
	 When we change the distribution from $P_i$ into $P'_i$, we are moving some probability mass from smaller values to larger values. Recall that the reward function $R(x, S)$ has a monotonicity property (Assumption~\ref{assum:monotone}): if $x$ and $x'$ are two vectors in $[0, 1]^m$ such that $x_i \le x_i'$ for all $i\in[m]$, then $R(x, S) \le R(x', S)$ for all $S\in\calF$. Therefore we have $r_P(S) \le r_{P'}(S)$ for all $S\in\calF$.
	 \vspace{6pt}
	
	Now we prove (ii).
	Without loss of generality, we assume $S = \{1, 2, \ldots, n\}$ ($n \le m$). 
	Let $P'' = P_1''\times \cdots \times P_m''$ be a distribution over $[0, 1]^m$ such that the CDF of $P_i''$ is the following:
	%whose CDF vector is the following:
	\begin{align} \label{eqn:exp-reward-property-inproof-1}
	F''_i(x) = \begin{cases}
	\max \{  F_{i}(x) - \Lambda_i, 0\}, & 0\le x < 1, \\
	1, & x = 1.
	\end{cases}
%	F^{P''}_{i, j} = \begin{cases}
%	\max\{F^P_{i, j} - \Lambda_i, 0\} &\quad 1\le j \le s_i - 1,\\
%	1 &\quad j = s_i.
%	\end{cases}
	\end{align}
	It is easy to see that $F''_i(x) \le F'_i(x)$ for all $i\in[m]$ and $x\in[0, 1]$.
	 %$0\le F^{P''}_{i, 1} \le \cdots \le F^{P''}_{i, s_i}=1$ for all $i\in[m]$ and that $F^{P''}_{i, j} \le F^{P'}_{i, j}$ for all $i\in[m], j\in[s_i]$.
	Thus from the result in (i) we have
	\begin{equation} \label{eqn:exp-reward-property-inproof-2}
	r_{P'}(S) \le r_{P''}(S).
	\end{equation}
	
	Let $\supp(P_i) = \{v_{i, 1}, v_{i, 2}, \ldots, v_{i, s_i}\}$ where $0\le v_{i, 1} < \cdots< v_{i, s_i}\le 1$.
	Define $P_S = P_1 \times P_2 \times \cdots \times P_n$, and define $P'_S$ and $P''_S$ similarly.
	Recall that the reward function $R(x, S)$ can be written as $R_S(x_S) = R_S(x_1, \ldots, x_n)$.
	Then we have
	\begin{align*}
	&r_{P''}(S) - r_P(S) \\
	=\,& \sum_{x_1, \ldots, x_n} R_S(x_1, \ldots, x_n) P''_S(x_1, \ldots, x_n) - \sum_{x_1, \ldots, x_n} R_S(x_1, \ldots, x_n) P_S(x_1, \ldots, x_n) \\
	=\,& \sum_{x_1, \ldots, x_n} R_S(x_1, \ldots, x_n) \cdot \left( P''_S(x_1, \ldots, x_n) - P_S(x_1, \ldots, x_n) \right) \\
	\le& \sum_{x_1, \ldots, x_n} M \cdot \left| P''_S(x_1, \ldots, x_n) - P_S(x_1, \ldots, x_n) \right| \\
	=\,& M \cdot L_1(P''_S, P_S),
	\end{align*}
	where the summation is taken over $x_i \in \{v_{i, 1}, \ldots, v_{i, s_i}\}$ ($\forall i\in S$).
	Then using Lemma~\ref{lem:l1-product-distribuction} we obtain
	\begin{equation} \label{eqn:exp-reward-property-inproof-3}
	r_{P''}(S) - r_P(S) \le M \cdot \sum_{i\in S} L_1(P''_i, P_i).
	\end{equation}
	
	Now we give an upper bound on $L_1(P''_i, P_i)$ for each $i$.
	Let $F_{i, j} = F_i(v_{i, j})$, $F''_{i, j} = F''_i(v_{i, j})$, and
	$F_{i, 0} = F''_{i, 0} = 0$. We have
	\begin{equation} \label{eqn:exp-reward-property-inproof-4}
	\begin{aligned}
	L_1(P''_i, P_i) &= \sum_{j=1}^{s_i} \left| P''_i(v_{i, j}) - P_i(v_{i, j}) \right| \\
	&= \sum_{j=1}^{s_i} \left|  (F''_{i, j} - F''_{i, j-1}) - (F_{i, j} - F_{i, j-1})  \right| \\
	&= \sum_{j=1}^{s_i} \left|  (F_{i, j} - F''_{i, j}) - (F_{i, j-1} - F''_{i, j-1})  \right|.
	\end{aligned}
	\end{equation}
	In fact, for all $1\le j<s_i$, we have $F_{i, j} - F''_{i, j} \ge F_{i, j-1} - F''_{i, j-1}$. To see this, consider two cases:
	\begin{itemize}
		\item If $F_{i, j} < \Lambda_i$, then we have $F_{i, j-1} \le F_{i, j} < \Lambda_i$.
		By definition \eqref{eqn:exp-reward-property-inproof-1} we have $F''_{i, j} = F''_{i, j-1} = 0$. Thus $F_{i, j} - F''_{i, j} = F_{i, j} \ge F_{i, j-1} = F_{i, j-1} - F''_{i, j-1}$.
		\item If $F_{i, j} \ge \Lambda_i$, then by definition \eqref{eqn:exp-reward-property-inproof-1} we have $F_{i, j} - F''_{i, j} = \Lambda_i \ge F_{i, j-1} - F''_{i, j-1}$.
	\end{itemize}
	Therefore \eqref{eqn:exp-reward-property-inproof-4} becomes
	\begin{equation} \label{eqn:exp-reward-property-inproof-5}
	\begin{aligned}
	L_1(P''_i, P_i) &= \sum_{j=1}^{s_i-1} \left(  (F_{i, j} - F''_{i, j}) - (F_{i, j-1} - F''_{i, j-1})  \right) + \left|  (1-1) - (F_{i, s_i-1} - F''_{i, s_i-1})  \right| \\
	&= F_{i, s_i-1} - F''_{i, s_i-1} + \left| F_{i, s_i-1} - F''_{i, s_i-1} \right| \\
	&= 2\left( F_{i, s_i-1} - F''_{i, s_i-1} \right) \\
	&\le 2\Lambda_i,
	\end{aligned}
	\end{equation}
	where the last inequality is due to \eqref{eqn:exp-reward-property-inproof-1}.
	
	We complete the proof of the lemma by combining \eqref{eqn:exp-reward-property-inproof-2}, \eqref{eqn:exp-reward-property-inproof-3} and \eqref{eqn:exp-reward-property-inproof-5}:
	\begin{equation*}
	r_{P'}(S) - r_P(S) \le r_{P''}(S) - r_P(S) \le M \cdot \sum_{i\in S} L_1(P''_i, P_i) \le  2M \sum_{i\in S}\Lambda_i.  \qedhere
	\end{equation*}
\end{proof}

The following lemma is similar to Lemma~1 in \cite{Kveton15}.
We will use some additional notation:
\begin{itemize}
	\item For $t\ge m+1$ and $i\in[m]$, let $T_{i, t}$ be the value of counter $T_i$ right after the $t$-th round of \texttt{SDCB}. In other words, $T_{i, t}$ is the number of observed outcomes from arm $i$ in the first $t$ rounds.
	\item Let $S_t$ be the super arm selected by \texttt{SDCB} in the $t$-th round.
\end{itemize}

\begin{lemma} \label{lem:lem1-in-kveton15}
	Define an event in each round $t$ ($m+1 \le t \le T$):
	\begin{equation} \label{eqn:event-H_t}
	\calH_t = \left\{ 0<  \Delta_{S_t} \le 4M \cdot \sum_{i\in S_t} \sqrt{\frac{3\ln t}{2T_{i, t-1}}} \right\}.
	\end{equation}
	Then the $\alpha$-approximation regret of \texttt{SDCB} in $T$ rounds is at most
	\[
	\E \left[ \sum_{t=m+1}^T \mathds{1}\{\calH_t\} \Delta_{S_t} \right] + \left( \frac{\pi^2}{3}  + 1 \right) \alpha Mm.
	\]
\end{lemma}
\begin{proof}
	%For simplicity, we let $F = F^D$ be the CDF vector of $D$.
	
	Let $F_i$ be the CDF of $D_i$.
	Let $\hat F_{i, l}$ be the empirical CDF of the first $l$ observations from arm $i$.
	For $m+1 \le t \le T$, 
	define an event
	\[
	\calE_t = \left\{ \text{there exists } i\in [m] \text{ such that } \sup_{x\in[0, 1]} \left| \hat F_{i, T_{i, t-1}}(x) - F_{i}(x) \right| \ge \sqrt{\frac{3\ln t}{2T_{i, t-1}}} \right\},
	\]
	which means that the empirical CDF $\hat F_i$ is not close enough to the true CDF $F_i$
	%be the event that there exists an empirical probability $\hat F_{i, j}$ which is an ``inaccurate'' estimate of the true CDF value $F_{i, j}$
	at the \emph{beginning} of the $t$-th round.
	%Note that we always have $\hat F_{i, s_i, l} = F_{i, s_i} = 1$ for all $i\in[m], l\ge1$.
	
	Recall that we have $S^* = \argmax_{S \in \calF} \{r_D(S)\}$ and $\Delta_S = \max\{ \alpha \cdot r_D(S^*) - r_D(S), 0 \}$ ($S \in \calF$).
	We bound the $\alpha$-approximation regret of \texttt{SDCB} as
	\begin{equation} \label{eqn:regret-decomp}
	\begin{aligned}
	\reg_{D, \alpha}^{\texttt{SDCB}}(T)  &= \sum_{t=1}^T \E\left[ \alpha \cdot r_D(S^*) -  r_D(S_t) \right]  \le \sum_{t=1}^T \E \left[ \Delta_{S_t} \right] \\
	&=
	\E\left[ \sum_{t=1}^m \Delta_{S_t} \right] + \E\left[ \sum_{t=m+1}^T \mathds{1}\{\calE_t\} \Delta_{S_t} \right] + \E\left[ \sum_{t=m+1}^T \mathds{1} \{\neg \calE_t\} \Delta_{S_t} \right],
	\end{aligned}
	\end{equation}
	where $\neg \calE_t$ is the complement of event $\calE_t$.
	
	We separately bound each term in \eqref{eqn:regret-decomp}.
	
	(a) the first term
	
	The first term in \eqref{eqn:regret-decomp} can be trivially bounded as 
	\begin{equation} \label{eqn:decomp-first-term}
	\E\left[ \sum_{t=1}^m \Delta_{S_t} \right] \le \sum_{t=1}^m \alpha \cdot r_D(S^*)  \le m \cdot \alpha M.
	\end{equation}
	
	(b) the second term
	
	By the DKW inequality
	%Chernoff-Hoeffding bound, 
	we know that for any $i \in [m], l \ge 1 , t \ge m+1$ we have
	\[
	\Pr\left[ \sup_{x\in[0, 1]} \left| \hat F_{i, l}(x) - F_{i}(x) \right| \ge \sqrt{\frac{3\ln t}{2l}} \right] \le 2e^{-2l\cdot \frac{3\ln t}{2l}} = 2e^{-3\ln t} = 2 t^{-3}.
	\]
	Therefore
	\begin{align*}
	\E\left[\sum_{t=m+1}^T \mathds{1}\{\calE_t\} \right]
	&\le \sum_{t=m+1}^T \sum_{i=1}^m \sum_{l=1}^{t-1} \Pr\left[ \left| \hat F_{i, j, l} - F_{i, j} \right| \ge \sqrt{\frac{3\ln t}{2l}} \right]\\
	&\le \sum_{t=m+1}^T \sum_{i=1}^m \sum_{l=1}^{t-1} 2 t^{-3} \\
	&\le 2 m \sum_{t=m+1}^T  t^{-2} \\
	&\le \frac{\pi^2}{3} m,
	\end{align*}
	and then the second term in \eqref{eqn:regret-decomp} can be bounded as
	\begin{equation} \label{eqn:decomp-second-term}
	\E\left[ \sum_{t=m+1}^T \mathds{1}\{\calE_t\} \Delta_{S_t} \right]
	\le \frac{\pi^2}{3} m \cdot \left(\alpha \cdot r_D(S^*) \right)
	\le \frac{\pi^2}{3}  \alpha Mm.
	\end{equation}
	
	(c) the third term
	
	We fix $t > m$ and first assume $\neg \calE_t$ happens. Let $c_i = \sqrt{\frac{3\ln t}{2T_{i, t-1}}}$ for each $i\in[m]$.
	Since $\neg \calE_t$ happens, we have 
	\begin{equation} \label{eqn:accurate-estimates}
		\left| \hat F_{i, T_{i, t-1}}(x) -  F_{i}(x) \right| < c_{i} \qquad \forall i\in[m], x\in[0, 1].
	\end{equation}
	Recall that in round $t$ of \texttt{SDCB} (Algorithm~\ref{alg:SDCB}), the input to the oracle is $\wideubar D = \wideubar D_1 \times \cdots \times \wideubar D_m$, where the CDF $\wideubar F_i$ of $\wideubar D_i$ is
	%Let $\wideubar F = (\wideubar F_{i, j})_{i\in[m], j\in[s_i]}$ be the input to the oracle in round $t$, which is the CDF vector of a distribution $\wideubar D \in \Sigma$. By definition, we have
	\begin{equation} \label{eqn:lcb-def}
	\wideubar F_{i}(x) = \begin{cases}
	\max\{\hat F_{i, T_{i, t-1}}(x) - c_i, 0\}, &\quad 0\le x<1,\\
	1, &\quad x=1.
	\end{cases}
	\end{equation}
	From \eqref{eqn:accurate-estimates} and \eqref{eqn:lcb-def} we know that $\wideubar F_{i}(x) \le F_{i}(x) \le \wideubar F_{i}(x) + 2c_i$ for all $i\in[m], x\in[0, 1]$.
	Thus, from Lemma~\ref{lem:exp-reward-property} (i) we have
	\begin{equation} \label{eqn:inproof-dominant-distribution-property-1}
	r_D(S) \le r_{\wideubar D}(S) \qquad \forall S\in\calF,
	\end{equation}
	and from Lemma~\ref{lem:exp-reward-property} (ii) we have
	\begin{equation} \label{eqn:inproof-dominant-distribution-property-2}
	r_{\wideubar D}(S) \le r_D(S) + 2M \sum_{i\in S} 2c_i \qquad \forall S\in\calF.
	\end{equation}
	Also, from the fact that the algorithm chooses $S_t$ in the $t$-th round, we have
	\begin{equation} \label{eqn:inproof-oracle}
	r_{\wideubar D}(S_t)
	\ge \alpha \cdot \max_{S \in \calF} \{ r_{\wideubar D}(S) \}
	\ge \alpha \cdot r_{\wideubar D}(S^*).
	\end{equation}
	
	From \eqref{eqn:inproof-dominant-distribution-property-1}, \eqref{eqn:inproof-dominant-distribution-property-2} and \eqref{eqn:inproof-oracle} we have
	\begin{align*}
	\alpha \cdot r_D(S^*)
	\le \alpha \cdot r_{\wideubar D}(S^*)
	\le r_{\wideubar D}(S_t)
	\le r_D(S_t) + 2M \sum_{i\in S_t} 2c_i,
	\end{align*}
	which implies
	\[
	\Delta_{S_t} \le 4M \sum_{i\in S_t} c_i.
	\]

	Therefore, when $\neg \calE_t$ happens, we always have $  \Delta_{S_t} \le 4M \sum_{i\in S_t} c_i $. In other words,
	\[
	\neg \calE_t \Longrightarrow \left\{ \Delta_{S_t} \le 4M \sum_{i\in S_t} \sqrt{\frac{3\ln t}{2T_{i, t-1}}} \right\}.
	\]
	This implies
	\[
	\{\neg {\calE}_t, \Delta_{S_t} > 0 \} \Longrightarrow \left\{0 < \Delta_{S_t} \le 4M \sum_{i\in S_t} \sqrt{\frac{3\ln t}{2T_{i, t-1}}} \right\} = \calH_t.
	\]
	Hence, the third term in \eqref{eqn:regret-decomp} can be bounded as
	\begin{equation} \label{eqn:decomp-third-term}
	\E\left[ \sum_{t=m+1}^T \mathds{1} \{\neg {\calE_t}\} \Delta_{S_t} \right] = \E\left[ \sum_{t=m+1}^T \mathds{1} \{\neg {\calE_t}, \Delta_{S_t} > 0\} \Delta_{S_t} \right] \le \E\left[ \sum_{t=m+1}^T \mathds{1} \{\mathcal H_t\} \Delta_{S_t} \right].
	\end{equation}
	
	Finally, by combining \eqref{eqn:regret-decomp}, \eqref{eqn:decomp-first-term}, \eqref{eqn:decomp-second-term} and \eqref{eqn:decomp-third-term} we have
	\begin{align*}
	\reg_{D, \alpha}^{\texttt{SDCB}}(T)
	\le \E\left[ \sum_{t=m+1}^T \mathds{1} \{\mathcal H_t\} \Delta_{S_t} \right] + \left( \frac{\pi^2}{3} + 1 \right) \alpha Mm,
	\end{align*}
	 completing the proof of the lemma.
\end{proof}

\subsubsection{Finishing the Proof of Theorem~\ref{thm:regret-bound-discrete}} \label{subsubsec:proof-thm-1}

Lemma~\ref{lem:lem1-in-kveton15} is very similar to Lemma 1 in~\cite{Kveton15}.
We now apply the counting argument in \cite{Kveton15} to finish the proof of Theorem~\ref{thm:regret-bound-discrete}. 

From Lemma~\ref{lem:lem1-in-kveton15} we know that it remains to bound $\E \left[ \sum_{t=m+1}^T \mathds{1}\{\calH_t\} \Delta_{S_t} \right] $, where $\calH_t$ is defined in~\eqref{eqn:event-H_t}.

Define two decreasing sequences of positive constants
\begin{align*}
1 = \beta_0 >& \beta_1 > \beta_2 > \ldots\\
& \alpha_1 > \alpha_2 > \ldots
\end{align*}
such that $\lim_{k\to\infty} \alpha_k = \lim_{k\to\infty} \beta_k = 0$. We choose $\{\alpha_k\}$ and $\{\beta_k\}$ as in Theorem 4 of \cite{Kveton15}, which satisfy
\begin{equation} \label{eqn:alpha-beta-property-1}
\sqrt 6 \sum_{k=1}^\infty \frac{\beta_{k-1} - \beta_k}{\sqrt{\alpha_k}} \le 1
\end{equation}
and
\begin{equation}\label{eqn:alpha-beta-property-2}
 \sum_{k=1}^\infty \frac{\alpha_k}{\beta_k} < 267.
\end{equation}

For $t\in \{m+1, \ldots, T\}$ and $k \in \Z_+$, let
\[
m_{k, t} = \begin{cases}
\alpha_k \left( \frac{2MK}{\Delta_{S_t}} \right)^2 \ln T &\quad \Delta_{S_t} > 0,\\
+\infty &\quad  \Delta_{S_t} = 0,
\end{cases} 
\]
and
\[
A_{k, t} = \{i \in S_t \mid T_{i, t-1} \le m_{k, t} \}.
\]
Then we define an event
\[
\G_{k, t} = \{|A_{k, t}| \ge \beta_k K\},
\]
which means ``in the $t$-th round, at least $\beta_k K$ arms in $S_t$ had been observed at most $m_{k, t}$ times.''

\begin{lemma} \label{lem:lem3-in-kveton15}
	In the $t$-th round ($m+1 \le t \le T$), if event $\calH_t$ happens, then there exists $k\in \Z_+$ such that
	event $\G_{k, t}$ happens.
\end{lemma}
\begin{proof}
	Assume that $\calH_t$ happens and that none of $\G_{1, t}, \G_{2, t}, \ldots$ happens.
	Then $|A_{k, t}| < \beta_k K$ for all $k\in \Z_+$.
	
	Let $A_{0, t} = S_t$ and
	 $\bar A_{k, t} = S_t \setminus A_{k, t} $ for $k\in \Z_+ \cup\{0\}$.
	 It is easy to see $\bar A_{k-1, t} \subseteq \bar A_{k, t}$ for all $k\in \Z_+$. Note that $\lim_{k\to\infty} m_{k, t} = 0$. Thus there exists $N \in \Z_+$ such that $\bar A_{k, t} = S_t$ for all $k\ge N$, and then we have $S_t = \bigcup_{k=1}^\infty \left( \bar A_{k, t} \setminus \bar A_{k-1, t} \right)$.
	 Finally, note that for all $i\in \bar A_{k, t}$, we have $T_{i, t-1} > m_{k, t}$. Therefore
	 \begin{align*}
	 \sum_{i\in S_t} \frac{1}{\sqrt{T_{i, t-1}}}
	 &= \sum_{k=1}^\infty \sum_{i\in \bar A_{k, t} \setminus \bar A_{k-1, t}} \frac{1}{\sqrt{T_{i, t-1}}}
	 \le \sum_{k=1}^\infty \sum_{i\in \bar A_{k, t} \setminus \bar A_{k-1, t}} \frac{1}{\sqrt{m_{k, t}}}\\
	 &= \sum_{k=1}^\infty \frac{\left| \bar A_{k, t} \setminus \bar A_{k-1, t} \right|}{\sqrt{m_{k, t}}}
	 = \sum_{k=1}^\infty \frac{\left|  A_{k-1, t} \setminus  A_{k, t} \right|}{\sqrt{m_{k, t}}}
	 = \sum_{k=1}^\infty \frac{\left|  A_{k-1, t} \right| - \left|  A_{k, t} \right|}{\sqrt{m_{k, t}}}\\
	 &= \frac{|S_t|}{\sqrt{m_{1, t}}} + \sum_{k=1}^\infty |A_{k, t}| \left( \frac{1}{\sqrt{m_{k+1, t}}} - \frac{1}{\sqrt{m_{k, t}}} \right)\\
	 &< \frac{K}{\sqrt{m_{1, t}}} + \sum_{k=1}^\infty \beta_k K \left( \frac{1}{\sqrt{m_{k+1, t}}} - \frac{1}{\sqrt{m_{k, t}}} \right)\\
	 &= \sum_{k=1}^\infty \frac{(\beta_{k-1} - \beta_k) K}{\sqrt{m_{k, t}}}.
	 \end{align*}
	 
	 Note that we assume $\calH_t$ happens. Then we have
	 \begin{align*}
	 \Delta_{S_t}
	 &\le 4M \cdot \sum_{i\in S_t} \sqrt{\frac{3\ln t}{2T_{i, t-1}}} 
	 \le 2M \sqrt {6\ln T}   \cdot \sum_{i\in S_t} \frac{1}{\sqrt{T_{i, t-1}}}\\
	 & < 2M \sqrt {6\ln T}  \cdot \sum_{k=1}^\infty \frac{(\beta_{k-1} - \beta_k) K}{\sqrt{m_{k, t}}}
	 = \sqrt 6 \sum_{k=1}^\infty \frac{\beta_{k-1} - \beta_k}{\sqrt{\alpha_k}} \cdot \Delta_{S_t} \le \Delta_{S_t},
	 \end{align*}
	 where the last inequality is due to~\eqref{eqn:alpha-beta-property-1}. We reach a contradiction here.
	 The proof of the lemma is completed.
\end{proof}

By Lemma~\ref{lem:lem3-in-kveton15} we have
\[
\sum_{t=m+1}^T \mathds{1}\{\calH_t\} \Delta_{S_t}
\le \sum_{k=1}^\infty \sum_{t=m+1}^T \mathds1 \{\G_{k, t}, \Delta_{S_t}>0\} \Delta_{S_t}.
\]
For $i\in[m], k\in\Z_+, t\in\{m+1, \ldots, T\}$, define an event
\[
\G_{i, k, t} = \G_{k, t} \wedge \{i\in S_t, T_{i, t-1} \le m_{k, t} \}.
\]
Then by the definitions of $\G_{k, t}$ and $\G_{i, k, t}$ we have
\[
\mathds1 \{\G_{k, t}, \Delta_{S_t}>0\}
\le \frac{1}{\beta_k K} \sum_{i\in E_\bad} \mathds1 \{\G_{i, k, t}, \Delta_{S_t}>0\}.
\]
Therefore
\[
\sum_{t=m+1}^T \mathds{1}\{\calH_t\} \Delta_{S_t}
\le \sum_{i\in E_\bad} \sum_{k=1}^\infty \sum_{t=m+1}^T \mathds1 \{\G_{i, k, t}, \Delta_{S_t}>0\} \frac{\Delta_{S_t}}{\beta_k K}.
\]

For each arm $i\in E_\bad$, suppose $i$ is contained in $N_i$ \emph{bad} super arms $S_{i, 1}^\bad, S_{i, 2}^\bad, \ldots, S_{i, N_i}^\bad$. Let $\Delta_{i, l} = \Delta_{S_{i, l}^\bad}$ ($l\in[N_i]$). Without loss of generality, we assume $\Delta_{i, 1} \ge \Delta_{i, 2} \ge \ldots \ge \Delta_{i, N_i}$.
Note that $\Delta_{i, N_i} = \Delta_{i, \min}$.
For convenience, we also define $\Delta_{i, 0} = +\infty$, i.e., $\alpha_k \left( \frac{2MK}{\Delta_{i, 0}} \right)^2 = 0$.
Then we have
\begin{align*}
&\sum_{t=m+1}^T \mathds{1}\{\calH_t\} \Delta_{S_t}\\
\le& \sum_{i\in E_\bad} \sum_{k=1}^\infty \sum_{t=m+1}^T \sum_{l=1}^{N_i} \mathds1 \{\G_{i, k, t}, S_t = S_{i, l}^\bad\} \frac{\Delta_{S_t}}{\beta_k K} \\
\le& \sum_{i\in E_\bad} \sum_{k=1}^\infty \sum_{t=m+1}^T \sum_{l=1}^{N_i} \mathds1 \{T_{i, t-1} \le m_{k, t}, S_t = S_{i, l}^\bad\} \frac{\Delta_{i, l}}{\beta_k K} \\
=& \sum_{i\in E_\bad} \sum_{k=1}^\infty \sum_{t=m+1}^T \sum_{l=1}^{N_i} \mathds1 \left\{T_{i, t-1} \le \alpha_k \left( \frac{2MK}{\Delta_{i, l}} \right)^2 \ln T, S_t = S_{i, l}^\bad \right\} \frac{\Delta_{i, l}}{\beta_k K} \\
=& \sum_{i\in E_\bad} \sum_{k=1}^\infty \sum_{t=m+1}^T \sum_{l=1}^{N_i} \sum_{j=1}^l \mathds1 \left\{\alpha_k \left( \frac{2MK}{\Delta_{i, j-1}} \right)^2 \ln T < T_{i, t-1} \le \alpha_k \left( \frac{2MK}{\Delta_{i, j}} \right)^2 \ln T, S_t = S_{i, l}^\bad \right\} \frac{\Delta_{i, l}}{\beta_k K} \\
\le& \sum_{i\in E_\bad} \sum_{k=1}^\infty \sum_{t=m+1}^T \sum_{l=1}^{N_i} \sum_{j=1}^l \mathds1 \left\{\alpha_k \left( \frac{2MK}{\Delta_{i, j-1}} \right)^2 \ln T < T_{i, t-1} \le \alpha_k \left( \frac{2MK}{\Delta_{i, j}} \right)^2 \ln T, S_t = S_{i, l}^\bad \right\} \frac{\Delta_{i, j}}{\beta_k K} \\
\le& \sum_{i\in E_\bad} \sum_{k=1}^\infty \sum_{t=m+1}^T \sum_{l=1}^{N_i} \sum_{j=1}^{N_i} \mathds1 \left\{\alpha_k \left( \frac{2MK}{\Delta_{i, j-1}} \right)^2 \ln T < T_{i, t-1} \le \alpha_k \left( \frac{2MK}{\Delta_{i, j}} \right)^2 \ln T, S_t = S_{i, l}^\bad \right\} \frac{\Delta_{i, j}}{\beta_k K} \\
\le& \sum_{i\in E_\bad} \sum_{k=1}^\infty \sum_{t=m+1}^T  \sum_{j=1}^{N_i} \mathds1 \left\{\alpha_k \left( \frac{2MK}{\Delta_{i, j-1}} \right)^2 \ln T < T_{i, t-1} \le \alpha_k \left( \frac{2MK}{\Delta_{i, j}} \right)^2 \ln T, i\in S_t \right\} \frac{\Delta_{i, j}}{\beta_k K} \\
\le& \sum_{i\in E_\bad} \sum_{k=1}^\infty \sum_{j=1}^{N_i} \left( \alpha_k \left( \frac{2MK}{\Delta_{i, j}} \right)^2 \ln T - \alpha_k \left( \frac{2MK}{\Delta_{i, j-1}} \right)^2 \ln T \right) \frac{\Delta_{i, j}}{\beta_k K}\\
=& 4M^2K \left(\sum_{k=1}^\infty \frac{\alpha_k}{\beta_k} \right) \ln T \cdot \sum_{i\in E_\bad}  \sum_{j=1}^{N_i} \left( \frac{1}{\Delta_{i, j}^2} - \frac{1}{\Delta_{i, j-1}^2} \right) \Delta_{i, j}\\
\le& 1068 M^2K  \ln T \cdot \sum_{i\in E_\bad}  \sum_{j=1}^{N_i} \left( \frac{1}{\Delta_{i, j}^2} - \frac{1}{\Delta_{i, j-1}^2} \right) \Delta_{i, j},
\end{align*}
where the last inequality is due to~\eqref{eqn:alpha-beta-property-2}.

Finally, for each $i\in E_\bad$ we have
\begin{align*}
\sum_{j=1}^{N_i} \left( \frac{1}{\Delta_{i, j}^2} - \frac{1}{\Delta_{i, j-1}^2} \right) \Delta_{i, j}
&= \frac{1}{\Delta_{i, N_i}} + \sum_{j=1}^{N_i - 1} \frac{1}{\Delta_{i, j}^2} (\Delta_{i, j} - \Delta_{i, j+1})\\
& \le \frac{1}{\Delta_{i, N_i}} + \int_{\Delta_{i, N_i}}^{\Delta_{i, 1}} \frac{1}{x^2} \dx\\
& = \frac{2}{\Delta_{i, N_i}} - \frac{1}{\Delta_{i, 1}}\\
& < \frac{2}{\Delta_{i, \min}}.
\end{align*}
It follows that
\begin{equation} \label{eqn:distri-dep-bound-inproof-1}
\sum_{t=m+1}^T \mathds{1}\{\calH_t\} \Delta_{S_t}
\le 1068 M^2K  \ln T \cdot \sum_{i\in E_\bad} \frac{2}{\Delta_{i, \min}}
=  M^2K  \sum_{i\in E_\bad} \frac{2136}{\Delta_{i, \min}} \ln T .
\end{equation}
Combining \eqref{eqn:distri-dep-bound-inproof-1} with Lemma~\ref{lem:lem1-in-kveton15}, the distribution-dependent regret bound in Theorem~\ref{thm:regret-bound-discrete} is proved.
\vspace{6pt}

To prove the distribution-independent bound, we decompose $\sum_{t=m+1}^T \mathds{1}\{\calH_t\} \Delta_{S_t}$ into two parts:
\begin{equation} \label{eqn:distri-dep-bound-inproof-2}
\begin{aligned}
\sum_{t=m+1}^T \mathds{1}\{\calH_t\} \Delta_{S_t}
&= \sum_{t=m+1}^T \mathds{1}\{\calH_t, \Delta_{S_t} \le \epsilon\} \Delta_{S_t} + \sum_{t=m+1}^T \mathds{1}\{\calH_t, \Delta_{S_t} > \epsilon\} \Delta_{S_t}\\
&\le \epsilon T + \sum_{t=m+1}^T \mathds{1}\{\calH_t, \Delta_{S_t} > \epsilon\} \Delta_{S_t},
\end{aligned}
\end{equation}
where $\epsilon > 0$ is a constant to be determined.
The second term can be bounded in the same way as in the proof of the distribution-dependent regret bound, except that we only consider the case $\Delta_{S_t} > \epsilon$.
Thus we can replace~\eqref{eqn:distri-dep-bound-inproof-1} by
\begin{equation} \label{eqn:distri-dep-bound-inproof-3}
\sum_{t=m+1}^T \mathds{1}\{\calH_t, \Delta_{S_t} > \epsilon\} \Delta_{S_t}
\le  M^2K  \sum_{i\in E_\bad , \Delta_{i, \min} > \epsilon} \frac{2136}{\Delta_{i, \min}} \ln T 
%&\le  M^2K  \sum_{i\in E_\bad} \frac{2136}{\epsilon} \ln T \\
\le M^2K m   \frac{2136}{\epsilon} \ln T.
\end{equation}
It follows that
\[
\sum_{t=m+1}^T \mathds{1}\{\calH_t\} \Delta_{S_t}
\le \epsilon T + M^2K m   \frac{2136}{\epsilon} \ln T.
\]
Finally, letting $\epsilon = \sqrt{\frac{2136 M^2 K m \ln T}{T}}$, we get
\[
\sum_{t=m+1}^T \mathds{1}\{\calH_t\} \Delta_{S_t}
\le 2\sqrt{2136 M^2 K m T \ln T}
< 93 M \sqrt{m KT \ln T}.
\]
Combining this with Lemma~\ref{lem:lem1-in-kveton15}, we conclude the proof of the distribution-independent regret bound in Theorem~\ref{thm:regret-bound-discrete}. \qed

\subsection{Analysis of Our Algorithm in the Previous CMAB Framework} \label{subsec:previous-cmab}

We now give an analysis of \texttt{SDCB} in the previous CMAB framework, following our discussion in Section~\ref{sec:alg}.
We consider the case in which the expected reward only depends on the means of the random variables. Namely, $r_D(S)$ only depends on $\mu_i$'s ($i\in S$), where $\mu_i$ is arm $i$'s mean outcome.
In this case, we can rewrite $r_D(S)$ as $r_\mu(S)$, where $\mu = (\mu_1, \ldots, \mu_m)$ is the vector of means.
Note that the offline computation oracle only needs a mean vector as input.

We no longer need the three assumptions (Assumptions~\ref{assum:indep}-\ref{assum:monotone}) given in Section~\ref{sec:problem}.
In particular,
 we do not require independence among outcome distributions of all arms (Assumption~\ref{assum:indep}).
 Although we cannot write $D$ as $D = D_1 \times \cdots \times D_m$,
we still let $D_i$ be the outcome distribution of arm $i$.
In this case, $D_i$ is the marginal distribution of $D$ in the $i$-th component.

\begin{algorithm}[t]
	\caption{\texttt{CUCB} \cite{chen2016cmab, Kveton15}}  
	\label{alg:CUCB}
	\begin{algorithmic}[1]
		%\REQUIRE $m$, $\calF$, $\supp(D_i) = \{v_{i, 1}, v_{i, 2}, \ldots, v_{i, s_i}\}$ for each $i\in[m]$
		%\vspace{3pt}
		\STATE For each arm $i$, maintain: (i) $\hat \mu_i$, the average of all observed outcomes from arm $i$ so far, and (ii) $T_i$, the number of observed outcomes from arm $i$ so far.
		\vspace{3pt}
		\STATE // Initialization
		\FOR{$i=1$ \TO $m$}
		\STATE // Action in the $i$-th round
		\STATE Play a super arm $S_i$ that contains arm $i$, and update $\hat \mu_i$ and $T_i$.
		\ENDFOR
		\vspace{6pt}
		
		\FOR{$t=m+1, m+2, \ldots$}
		\STATE // Action in the $t$-th round
		%\FOR{$i=1, 2, \ldots, m$}
		\STATE $\bar \mu_i \leftarrow \min\{ \hat \mu_i + \sqrt{\frac{3\ln t}{2T_i}}, 1\} \qquad \forall i\in [m]$ \label{line:cucb-ucb}
		%\ENDFOR
		\STATE Play the super arm $S_t \leftarrow \text{Oracle}(\bar{\mu})$, where $\bar{\mu}  = ( \bar\mu_1, \ldots, \bar\mu_m)$.
		\STATE Update $\hat{\mu}_i$ and $T_i$ for all $i\in S_t$.
		\ENDFOR
	\end{algorithmic}
\end{algorithm}

We summarize the \texttt{CUCB} algorithm \cite{chen2016cmab, Kveton15} in Algorithm~\ref{alg:CUCB}.
It maintains the empirical mean $\hat \mu_i$ of the outcomes from each arm $i$, and stores the number of observed outcomes from arm $i$ in a variable $T_i$.
In each round, it calculates an upper confidence bound (UCB) $\bar \mu_i$ of $\mu_i$,
Then it uses the UCB vector $\bar \mu$ as the input to the oracle, and plays the super arm output by the oracle.
In the $t$-th round ($t>m$),
each UCB $\bar \mu_i$ has the key property that 
\begin{equation} \label{eqn:ucb-property}
\mu_i \le \bar \mu_i \le \mu_i + 2 \sqrt{\frac{3\ln t}{2T_{i, t-1}}}
\end{equation}
holds with high probability. 
(Recall that $T_{i,t-1}$ is the value of $T_i$ after $t-1$ rounds.)
To see this, note that we have $|\mu_i - \hat\mu_i| \le \sqrt{\frac{3\ln t}{2T_{i, t-1}}}$ with high probability (by Chernoff bound), and then \eqref{eqn:ucb-property} follows from the definition of $\bar \mu_i$ in line~\ref{line:cucb-ucb} of Algorithm~\ref{alg:CUCB}.

We prove that the same property as \eqref{eqn:ucb-property} also holds for \texttt{SDCB}.
Consider a fixed $t>m$, and let %$\wideubar F = (\wideubar F_{i,j})$ be the CDF vector used as
$\wideubar D = \wideubar D_1 \times \cdots \times \wideubar D_m$ be the input to the oracle in the $t$-th round of  \texttt{SDCB}. 
%Let $\wideubar D_i$ be the distribution corresponding to the CDF values $(\wideubar F_{i, j})_{j\in[s_i]}$, i.e., $\wideubar D_i$ takes values from $\{v_{i, 1}, \ldots, v_{i, s_i}\}$ and $\Pr_{Y_i \sim \wideubar D_i} [Y_i \le v_{i, j}] = \wideubar F_{i, j}$ for all $j\in[s_i]$.
%Suppose $\wideubar F$ is the CDF vector of the distribution $\wideubar D = \wideubar D_1 \times \cdots \times \wideubar D_m$.
Let $\nu_i = \E_{Y_i \sim \wideubar D_i} [Y_i]$.
We can think that \texttt{SDCB} uses the mean vector $\nu = (\nu_1, \ldots, \nu_m)$ as the input to the oracle used by \texttt{CUCB}.
We now show that for each $i$, we have
\begin{equation} \label{eqn:sdcb-mean-property}
\mu_i \le \nu_i \le \mu_i + 2 \sqrt{\frac{3\ln t}{2T_{i, t-1}}}
\end{equation}
with high probability.

To show \eqref{eqn:sdcb-mean-property}, 
 we first prove the following lemma.
\begin{lemma} \label{lem:sdcb-mean-property}
	Let $P$ and $P'$ be two distributions over $[0, 1]$ %taking values from the set $\{v_1, v_2, \ldots, v_s\}$ where $0 \le v_1 < v_2 < \cdots < v_s \le 1$.
	with CDFs $F$ and $F'$, respectively.
	Consider two random variables $Y \sim P$ and $Y' \sim P'$.
	%Let $F_j = \Pr[Y \le v_{j}]$ and $F_j' = \Pr[Y' \le v_{ j}]$, for each $j\in[s]$.
	\begin{enumerate}[(i)]
		\item If for all $x\in[0, 1]$ we have $F'(x) \le F(x)$, then we have
		$
		\E[Y] \le \E[Y']
		$.
		\item If for all $x\in[0, 1]$ we have $F(x) - F'(x) \le \Lambda$ ($\Lambda>0$), then we have
		$
		\E[Y'] \le \E[Y] + \Lambda
		$.
	\end{enumerate}
\end{lemma}
\begin{proof}
	%Let $F_0 = 0$.
	We have
	\begin{align*}
	\E [Y] = \int_0^1 x \dFx = (xF(x))\big|_0^1 - \int_0^1 F(x) \dx = 1 - \int_0^1 F(x) \dx.
	%&= \sum_{j=1}^s v_j \cdot \Pr[Y = v_j]
	%= \sum_{j=1}^s v_j  (F_j - F_{j-1})
	%= \sum_{j=1}^{s-1} F_j(v_j - v_{j+1}) + v_sF_s \\
	%&= \sum_{j=1}^{s-1} F_j(v_j - v_{j+1}) + v_s.
	\end{align*}
	Similarly, we have
	\[
	\E [Y'] = 1 - \int_0^1 F'(x) \dx.
	\]
	Then the lemma holds trivially.
%	Therefore, we have
%	\begin{equation} \label{eqn:expectation-diff}
%	\E[Y'] - \E[Y] = \sum_{j=1}^{s-1} (F_j - F_j') (v_{j+1} - v_{j}).
%	\end{equation}
%	
%	(i) If $F_j' \le F_j$ for all $j\in[s]$, then form~\eqref{eqn:expectation-diff} we have $\E[Y'] - \E[Y] \ge \sum_{j=1}^{s-1} 0\cdot (v_{j+1} - v_{j}) = 0$.
%	
%	(ii) If $F_j - F_j' \le \Lambda$ for all $j\in[s]$, then form~\eqref{eqn:expectation-diff} we have $\E[Y'] - \E[Y] \le \sum_{j=1}^{s-1} \Lambda (v_{j+1} - v_{j}) = \Lambda (v_s - v_1) \le \Lambda$.
\end{proof}

Now we prove \eqref{eqn:sdcb-mean-property}.
According to the DKW inequality, %$\wideubar F_{i}$ used by
%\texttt{SDCB} is a lower confidence bound on the CDF value $F_{i, j} = \Pr_{X_i \sim D_i}[X_i \le v_{i, j}]$ of the true outcome distribution $D_i$.
with high probability we have 
\begin{equation} \label{eqn:sdcb-property}
F_{i}(x) - 2 \sqrt{\frac{3\ln t}{2T_{i, t-1}}} \le \wideubar F_{i}(x)  \le F_{i}(x)
\end{equation}
 for all $i\in[m]$ and $x\in[0, 1]$, where $\wideubar F_i$ is the CDF of $\wideubar D_i$ used in round $t$ of \texttt{SDCB}, and $F_i$ is the CDF of $D_i$.
Suppose \eqref{eqn:sdcb-property} holds for all $i,x$, then for any $i$, the two distributions $D_i$ and $\wideubar D_i$ satisfy the two conditions in Lemma~\ref{lem:sdcb-mean-property}, with $\Lambda = 2 \sqrt{\frac{3\ln t}{2T_{i, t-1}}}$;
then from Lemma~\ref{lem:sdcb-mean-property} we know that $\mu_i \le \nu_i \le \mu_i + 2 \sqrt{\frac{3\ln t}{2T_{i, t-1}}}$.
Hence we have shown that \eqref{eqn:sdcb-mean-property} holds with high probability.

The fact that \eqref{eqn:sdcb-mean-property} holds with high probability means that the mean of $\wideubar D_i$ is also a UCB of $\mu_i$ with the same confidence as in \texttt{CUCB}.
With this property, the analysis in~\cite{chen2016cmab, Kveton15} can also be applied to \texttt{SDCB}, resulting in exactly the same regret bounds.
%The only difference is that the failure probability bound of \eqref{eqn:sdcb-mean-property} in \texttt{SDCB} is larger than the failure probability bound of \eqref{eqn:ucb-property} in \texttt{CUCB} by a factor of $s$, due to a union bound for all $s_i$ supported points;
%thus the constant term (i.e., the term not containing $T$) in the regret bounds of \texttt{SDCB} has an extra term of $s$.
%The term containing $T$ in the regret bounds of \texttt{CUCB} and \texttt{SDCB} will be exactly the same.

\section{Missing Proofs from Section~\ref{sec:discretization}} \label{appdx:proof-general-case}

\subsection{Analysis of the Discretization Error} \label{subsec:disc-error}

The following lemma gives an upper bound on the error due to discretization.
Refer to Section~\ref{sec:discretization} for the definition of the discretized distribution $\tilde D$.

\begin{lemma} \label{lem:disc-error}
	For any $S \in \calF$, we have \[ \left| r_D(S) - r_{\tilde{D}}(S) \right| \le \frac{CK}{s}. \]
\end{lemma}

To prove Lemma~\ref{lem:disc-error},
 we show a slightly more general lemma which gives an upper bound on the discretization error of the expectation of a Lipschitz continuous function.

\begin{lemma} \label{lem:disc-error-general}
	Let $g(x)$ be a Lipschitz continuous function on $[0, 1]^n$ such that for any $x, x' \in [0, 1]^n$, we have $|g(x) - g(x')| \le C \|x - x' \|_1$, where $\|x- x' \|_1 = \sum_{i=1}^n |x_i - x_i'|$.
	Let $P = P_1  \times \cdots \times P_n$ be a probability distribution over $[0, 1]^n$. %such that for $X = (X_1, X_2, \ldots, X_n) \sim P$, $X_1, X_2, \ldots, X_n$ are mutually independent.
	Define another distribution $\tilde P = \tilde P_1 \times \cdots \times \tilde P_n$ over $[0, 1]^n$ as follows: %(i) for $\tilde X = (\tilde X_1, \tilde X_2, \ldots, \tilde X_n) \sim \tilde P$, $\tilde X_1, \ldots, \tilde X_n$ are mutually independent, and (ii) 
	each $\tilde P_i$ ($i\in[n]$) takes values in $\{ \frac1s, \frac2s, \ldots, 1\}$, and
	\[
	\Pr_{\tilde{X}_i \sim \tilde{P}_i} [\tilde{X}_i =  j/s ] = \Pr_{X_i \sim P_i}\left[ X_i \in I_j \right], \qquad j \in [s],
	\]
	where $I_1 = [0, \frac1s], I_2 = (\frac1s, \frac2s], \ldots, I_{s-1} = (\frac{s-2}{s}, \frac{s-1}{s}],I_s = (\frac{s-1}{s}, 1] $.
	Then
	\begin{equation} \label{eqn:disc-error-toshow}
	\left| \E_{X \sim P} [ g(X) ] - \E_{\tilde X \sim \tilde P} [ g(\tilde X) ] \right|
	\le \frac{C\cdot n}{s}.
	\end{equation}
\end{lemma}
\begin{proof}
	Throughout the proof, we consider $X = (X_1, \ldots, X_n) \sim P$ and $\tilde X = (\tilde X_1, \ldots, \tilde X_n) \sim \tilde P$.
	
	Let $v_j = \frac js$ ($j = 0, 1, \ldots, s$) and
	\[
	p_{i, j} = \Pr [\tilde{X}_i = v_j ] = \Pr [ X_i \in I_j ] \qquad i\in[n], j\in[s].
	\]
	
	We prove \eqref{eqn:disc-error-toshow} by induction on $n$. \vspace{6pt}
	
	(1) When $n=1$, we have
	\begin{equation} \label{eqn:disc-error-inproof-1}
	\E [ g(X_1) ]
	%= \sum_{j=1}^s \Pr\left[ X_1 \in I_j \right] \cdot \E\left[g(X_1) \big| X_1 \in I_j \right]
	= \sum_{j\in[s], p_{1, j}>0} p_{1, j} \cdot \E \left[g(X_1) \big| X_1 \in I_j \right].
	\end{equation}
	
	Since $g$ is continuous, for each $j\in[s]$ such that $p_{1, j}>0$, there exists $\xi_j \in [v_{j-1}, v_j]$ such that
	\[
	\E\left[g(X_1) | X_1 \in I_j \right] = g(\xi_j)
	\]
	From the Lipschitz continuity of $g$ we have
	\[
	\left| g(v_j) - g(\xi_j) \right| \le C |v_j - \xi_j| \le C|v_j - v_{j-1}| = \frac Cs.
	\]
	Hence
	\begin{align*}
	\left| \E [ g(X_1) ] - \E [ g(\tilde X_1) ] \right|
	&= \left| \sum_{j\in[s], p_{1, j}>0} p_{1, j} \cdot \E [g(X_1) | X_1 \in I_j ] - \sum_{j\in[s], p_{1, j}>0} p_{1, j} \cdot g(v_j) \right|\\
	&= \left| \sum_{j\in[s], p_{1, j}>0} p_{1, j} \cdot  g(\xi_j) - \sum_{j\in[s], p_{1, j}>0} p_{1, j} \cdot g(v_j) \right|\\
	&\le \sum_{j\in[s], p_{1, j}>0} p_{1, j} \cdot \left| g(\xi_j) - g(v_{j}) \right| \\
	&\le \sum_{j\in[s], p_{1, j}>0} p_{1, j} \cdot \frac Cs\\
	&= \frac Cs.
	\end{align*}
	This proves \eqref{eqn:disc-error-toshow} for $n=1$. \vspace{6pt}
	
	(ii) Suppose \eqref{eqn:disc-error-toshow} is correct for $n = 1, 2, \ldots, k-1$. Now we prove it for $n=k$ ($k\ge 2$).
	
	We define two functions on $[0, 1]^{k-1}$:
	\[
	h(x_1, \ldots, x_{k-1}) = \E_{X_k} [g(x_1, \ldots, x_{k-1}, X_k)]
	\]
	and
	\[
	\tilde h(x_1, \ldots, x_{k-1}) = \E_{\tilde X_k} [g(x_1, \ldots, x_{k-1}, \tilde X_k)].
	\]
	For any fixed $x_1, \ldots, x_{k-1} \in [0, 1]$, the function $ g(x_1, \ldots, x_{k-1}, x)$ on $x\in [0, 1]$ is Lipschitz continuous. Therefore from the result for $n=1$ we have
	\[
	\left|h(x_1, \ldots, x_{k-1}) - \tilde h(x_1, \ldots, x_{k-1})\right| \le \frac Cs \qquad \forall x_1, \ldots, x_{k-1} \in [0, 1].
	\]
	Then we have
	\begin{equation}  \label{eqn:disc-error-inproof-2}
	\begin{aligned}  
	&\left| \E [g(X)] - \E[g(\tilde X)] \right| \\
	=\,& \left| \E_{X_1, \ldots, X_{k-1}} \left[\E[g(X) | X_1, \ldots, X_{k-1}] \right] - \E[g(\tilde X)] \right| \\
	=\,& \left| \E_{X_1, \ldots, X_{k-1}} \left[h(X_1, \ldots, X_{k-1}) \right] - \E[g(\tilde X)] \right| \\
	\le \,& \left| \E_{X_1, \ldots, X_{k-1}} [h(X_1, \ldots, X_{k-1})] - \E_{X_1, \ldots, X_{k-1}} [\tilde h(X_1, \ldots, X_{k-1})] \right| \\ &+ \left| \E_{X_1, \ldots, X_{k-1}} [\tilde h(X_1, \ldots, X_{k-1})] - \E[g(\tilde X)] \right| \\
	\le\,& \E_{X_1, \ldots, X_{k-1}} \left[\left| h(X_1, \ldots, X_{k-1}) - \tilde h(X_1, \ldots, X_{k-1}) \right|\right] \\&+ \left| \E_{X_1, \ldots, X_{k-1}, \tilde{X}_k} [g(X_1, \ldots, X_{k-1}, \tilde{X}_k)] - \E[g(\tilde X)] \right| \\
	\le\,& \E_{X_1, \ldots, X_{k-1}} \left[ \frac Cs \right] + \left| \E_{ \tilde{X}_k} \left[ \E [g(X_1, \ldots, X_{k-1}, \tilde{X}_k) | \tilde{X}_k] - \E  [g(\tilde X_1, \ldots, \tilde X_{k-1}, \tilde{X}_k) | \tilde X_k ] \right] \right| \\
	\le\,& \frac Cs +  \E_{ \tilde{X}_k} \left[ \left| \E [g(X_1, \ldots, X_{k-1}, \tilde{X}_k) | \tilde{X}_k] - \E  [g(\tilde X_1, \ldots, \tilde X_{k-1}, \tilde{X}_k) | \tilde X_k ] \right| \right] \\
	=\,& \frac Cs + \sum_{j\in[s], p_{k, j} > 0} p_{k, j} \cdot \left| \E [g(X_1, \ldots, X_{k-1}, v_j)] - \E [g(\tilde X_1, \ldots, \tilde X_{k-1}, v_j)] \right|. 
	\end{aligned}
	\end{equation}

	For any $j\in[s]$, the function $g(x_1, \ldots, x_{k-1}, v_j)$ on $(x_1, \ldots, x_{k-1}) \in [0, 1]^{k-1}$ is Lipschitz continuous. Then from the induction hypothesis at $n=k-1$, we have
	\begin{equation} \label{eqn:disc-error-inproof-3}
	\left| \E [g(X_1, \ldots, X_{k-1}, v_j)] - \E [g(\tilde X_1, \ldots, \tilde X_{k-1}, v_j)] \right| \le \frac{C(k-1)}{s} \qquad \forall j\in[s].
	\end{equation}
	From \eqref{eqn:disc-error-inproof-2} and \eqref{eqn:disc-error-inproof-3} we have
	\begin{align*}
	\left| \E [g(X)] - \E[g(\tilde X)] \right|
	&\le \frac Cs + \sum_{j\in[s], p_{k, j} > 0} p_{k, j} \cdot \frac{C(k-1)}{s}\\
	&= \frac Cs + \frac{C(k-1)}{s}\\
	&= \frac{Ck}{s}.
	\end{align*}
	This concludes the proof for $n=k$.
\end{proof}

Now we prove Lemma~\ref{lem:disc-error}. 

\begin{proof}[Proof of Lemma~\ref{lem:disc-error}]
	We have
	\[
	r_D(S) = \E_{X \sim D} [R(X, S)] = \E_{X \sim D} [R_S(X_S)] = \E_{X_S \sim D_S} [R_S(X_S)],
	\]
	where $X_S = (X_i)_{i\in S}$ and $D_S = (D_i)_{i\in S}$.
	Similarly, we have
	\[
	r_{\tilde D}(S) = \E_{\tilde X_S \sim \tilde D_S} [R_S(\tilde X_S)].
	\]
	According to Assumption~\ref{assum:lipschitz}, the function $R_S$ defined on $[0, 1]^{S}$ is Lipschitz continuous. Then from Lemma~\ref{lem:disc-error-general} we have
	\begin{align*}
	\left| r_D(S) - r_{\tilde D}(S) \right|
	= \left| \E_{X_S \sim D_S} [R_S(X_S)] - \E_{\tilde X_S \sim \tilde D_S} [R_S(\tilde X_S)] \right|
	\le \frac{C \cdot |S|}{s}
	\le \frac{C\cdot K}{s}.
	\end{align*}
	This completes the proof.
\end{proof}

\subsection{Proof of Theorem~\ref{thm:bound-know-T}}

\begin{proof}[Proof of Theorem~\ref{thm:bound-know-T}]

	Let ${S}^* = \argmax_{S\in \calF} \{r_{{D}}(S)\}$ and $\tilde{S}^* = \argmax_{S\in \calF} \{r_{\tilde{D}}(S)\}$ be the optimal super arms in problems $([m], \calF, {D}, R)$ and $([m], \calF, \tilde{D}, R)$, respectively.
	%For simplicity,	we refer to \texttt{SDCB-GDT} as $\A$.
	Suppose Algorithm~\ref{alg:Lazy-SDCB} selects super arm $S_t$ in the $t$-th round $(1\le t \le T)$. Then its $\alpha$-approximation regret is bounded as
	\begin{align*}
	&\reg_{D, \alpha}^{\text{Alg.~\ref{alg:Lazy-SDCB}}}(T) \\
	=\, &  T \cdot \alpha \cdot r_D(S^*) - \sum_{t=1}^T \E \left[r_D(S_t) \right]\\
	=\, & T \cdot \alpha \left( r_D(S^*) - r_{\tilde{D}}(\tilde S^*) \right)
	+ \sum_{t=1}^T \E \left[  r_{\tilde{D}}(S_t) - r_D(S_t) \right] + \left( T \cdot \alpha \cdot r_{\tilde{D}}(\tilde S^*) - \sum_{t=1}^T \E \left[r_{\tilde D}(S_t) \right] \right)\\
	\le\, & T \cdot \alpha \left( r_D(S^*) - r_{\tilde{D}}(S^*) \right) + \sum_{t=1}^T \E \left[  r_{\tilde{D}}(S_t) - r_D(S_t) \right] + \reg_{\tilde D, \alpha}^{\text{Alg.~\ref{alg:SDCB}}}(T).
	\end{align*}
	where the inequality is due to $r_{\tilde{D}}(\tilde S^*) \ge r_{\tilde{D}}( S^*)$.
	
	Then from Lemma~\ref{lem:disc-error} and the distribution-independent bound in Theorem~\ref{thm:regret-bound-discrete}	we have
	\begin{equation} \label{eqn:general-case-inproof-1}
	\begin{aligned}
	\reg_{D, \alpha}^{\text{Alg.~\ref{alg:Lazy-SDCB}}}(T) 
	&\le  T \cdot \alpha \cdot \frac{CK}{s} + T \cdot \frac{CK}{s} 
	+ 93 M  \sqrt{mK T \ln T} + \left( \frac{\pi^2}{3} + 1 \right) \alpha Mm \\
	&\le 2 \cdot \frac{CKT}{s} 
	+ 93 M  \sqrt{mK T \ln T} + \left( \frac{\pi^2}{3} + 1 \right) \alpha Mm \\
	&\le 93 M  \sqrt{mK T \ln T} + 2CK\sqrt T + \left( \frac{\pi^2}{3} + 1 \right) \alpha Mm.
	\end{aligned}
	\end{equation}
	Here in the last two inequalities we have used $\alpha\le 1$ and $s = \lceil \sqrt T \rceil \ge \sqrt T$.
	The proof is completed.

%\subsubsection{Distribution-Independent Bound}

%	Plugging the distribution-independent bound in Theorem~\ref{thm:regret-bound-discrete} into~\eqref{eqn:general-case-inproof-1}, we have
%	\begin{equation*}
%	\reg_{D, \alpha}^{\texttt{SDCB-GDT}}(T)
%	\le 	93 M  \sqrt{mK T \ln (\lambda T)} + \left( \frac{\pi^2}{3} \lambda^{-3} (s-1) + 1 \right) \alpha Mm + 2CKTs^{-1}.
%	\end{equation*}
%	Recall that in Algorithm~\ref{alg:SDCB-GDT} we have $s= \lceil T^{1+\eta} \rceil \ge T^{1+\eta}$ and $\lambda = (s-1)^{1/3} < T^{(1+\eta)/3}$. Therefore
%	\begin{align*}
%	\reg_{D, \alpha}^{\texttt{SDCB-GDT}}(T)
%	&\le 	93 M  \sqrt{mK T \ln T^{(4+\eta)/3}} + \left( \frac{\pi^2}{3} + 1 \right) \alpha Mm + \frac{2CK}{T^\eta}\\
%	%&\le 93 M \sqrt{ \frac{4+\eta}{3} mK T \ln T} + \left( \frac{\pi^2}{3} + 1 \right) \alpha Mm + \frac{2CK}{T^\eta}\\
%	&\le 54 M \sqrt{ (4+\eta) mK T \ln T} + \left( \frac{\pi^2}{3} + 1 \right)  \alpha Mm + \frac{2CK}{T^\eta}. \qedhere
%	\end{align*}

\end{proof}

\subsection{Proof of Theorem~\ref{thm:bound-not-know-T}}

\begin{proof}[Proof of Theorem~\ref{thm:bound-not-know-T}]
	Let $n = \lceil \log_2 T \rceil$. Then we have $2^{n-1} <T \le 2^n$.
	
	If $n \le q = \left\lceil \log_2 m  \right\rceil$, then $T \le 2m$ and the regret in $T$ rounds is at most $2m \cdot \alpha M$.
	The regret bound holds trivially.
	
	Now we assume $n \ge q+1$.
	Using Theorem~\ref{thm:bound-know-T}, we have
	\begin{align*}
	&\reg_{D, \alpha}^{\text{Alg.~\ref{alg:Lazy-SDCB-unknown-T}}}(T) \\
	\le\, & \reg_{D, \alpha}^{\text{Alg.~\ref{alg:Lazy-SDCB-unknown-T}}}(2^n)\\
	=\, & \reg_{D, \alpha}^{\text{Alg.~\ref{alg:Lazy-SDCB}}}(2^q) + \sum_{k=q}^{n-1} \reg_{D, \alpha}^{\text{Alg.~\ref{alg:Lazy-SDCB}}}(2^k)\\
	\le\, & \reg_{D, \alpha}^{\text{Alg.~\ref{alg:Lazy-SDCB}}}(2m) + \sum_{k=q}^{n-1} \reg_{D, \alpha}^{\text{Alg.~\ref{alg:Lazy-SDCB}}}(2^k) \\
	\le\, &  2m \cdot  \alpha M + \sum_{k=q}^{n-1} \left(  93 M  \sqrt{mK  \cdot 2^k \ln 2^k} + 2CK\sqrt{2^k} + \left( \frac{\pi^2}{3} + 1 \right) \alpha Mm \right) \\
	\le\, &  2  \alpha Mm  + \left( 93M \sqrt{ mK  \ln 2^{n-1}} + 2CK \right) \cdot \sum_{k=1}^{n-1} \sqrt{2^k} + (n-1) \cdot \left( \frac{\pi^2}{3} + 1 \right)  \alpha Mm  \\
	\le\, &   \left( 93M \sqrt{ mK  \ln 2^{n-1}} + 2CK \right) \cdot \frac{\sqrt{2^n}}{\sqrt2 - 1} + \left( \frac{\pi^2}{3} + 3 \right) (n-1) \cdot  \alpha Mm  \\
	\le\, &   \left( 93M \sqrt{ mK  \ln T} + 2CK \right)  \cdot \frac{\sqrt{2T}}{\sqrt2 - 1} + \left( \frac{\pi^2}{3} + 3 \right)  \log_2T \cdot  \alpha Mm \\
	\le\, & 318M \sqrt{mKT  \ln T} + 7 CK \sqrt{T} + 10 \alpha Mm \ln T . \qedhere
	\end{align*}
	%Note that we have used $2^{n-1} < T$.
\end{proof}

\section{The Offline $K$-MAX Problem} \label{appdx:offline-kmax}

In this section, we consider the offline $K$-MAX problem.
Recall that we have $m$ independent random variables 
$\{X_i\}_{i\in [m]}$.
$X_i$ follows the discrete distribution $D_i$ with support $\{v_{i,1}, \ldots, v_{i, s_i}\} \subset [0, 1]$, and $D = D_1 \times \cdots \times  D_m$ is the joint distribution of $X = (X_1, \ldots, X_m)$.
Let $p_{i,j} = \Pr [X_i = v_{i,j}]$.
Define $r_D(S) = \E_{X\sim D}[\max_{i\in S} X_i]$
and $\OPT=\max_{S: |S|=K} r_D(S)$.
Our goal is to find (in polynomial time) a subset $S \subseteq[m]$ of cardinality $K$ 
such that $r_D(S)\geq \alpha \cdot \OPT$ (for certain constant $\alpha$).

First, we show that $r_D(S)$ can be calculated in polynomial time given any $S \subseteq [m]$.
Let $S = \{i_1, i_2, \ldots, i_n\}$.
	Note that for $X\sim D$, $\max_{i\in S} X_i$ can only take values in the set $V(S) = \bigcup_{i\in S} \supp(D_{i})$. For any $v\in V(S)$, we have
	\begin{equation} \label{eqn:max-value-prob}
	\begin{aligned}
	&\Pr_{X\sim D} \left[\max_{i\in S} X_i = v \right]\\
	=\,& \Pr_{X\sim D} \left[X_{i_1} = v, X_{i_2} \le v, \ldots, X_{i_n} \le v \right]\\
	& + \Pr_{X\sim D}[X_{i_1} < v, X_{i_2} = v, X_{i_3}\le v, \ldots, X_{i_n} \le v]\\
	& + \cdots \\
	& + \Pr_{X\sim D}[X_{i_1} < v, \ldots, X_{i_{n-1}} < v, X_{i_n} = v].
	\end{aligned}
	\end{equation}
	Since $X_{i_1}, \ldots, X_{i_n}$ are mutually independent,
	each probability appearing in \eqref{eqn:max-value-prob} can be calculated in polynomial time. Hence for any $v\in V(S)$, $\Pr_{X\sim D} \left[\max_{i\in S} X_i = v \right]$ can be calculated in polynomial time using \eqref{eqn:max-value-prob}.
	Then $r_D(S)$ can be calculated by
	\[
	r_D(S) = \sum_{v\in V(S)} v\cdot \Pr_{X\sim D} \left[\max_{i\in S} X_i = v \right]
	\]
	in polynomial time.

\subsection{$(1-1/e)$-Approximation}

We now show that a simple greedy algorithm (Algorithm~\ref{alg:greedy-kmax}) can find a $(1-1/e)$-approximate solution, by proving the submodularity of $r_D(S)$.
In fact, this is implied by a slightly more general result
\cite[Lemma 3.2]{goel2006asking}. We provide a simple and direct 
proof for completeness.

\begin{algorithm}[t]
	\caption{\texttt{Greedy-K-MAX}}
	\label{alg:greedy-kmax}
	\begin{algorithmic}[1]
		\STATE $S \leftarrow \emptyset$
		\FOR{$i=1$ \TO $K$}
		\STATE $k \leftarrow \argmax_{j \in [m] \setminus S } r_D(S \cup \{j\})$
		\STATE $S \leftarrow S \cup \{k\}$
		\ENDFOR
		\ENSURE $S$
	\end{algorithmic}
\end{algorithm}

\begin{lemma}
	\label{lem:max_const}
	%Suppose each $D_i$ ($i\in[m]$) has a known finite support $\supp(D_i) = \{v_{i, 1}, v_{i, 2}, \ldots, v_{i, s_i}\}$. Then there is a polynomial-time algorithm\footnote{The running time of this algorithm is polynomial in $m$ and $\max_{i\in[m]} s_i$.} which,
	%given the probabilities $p_{i, j} = \Pr_{X_i \sim D_i}[X_i = v_{i, j}]$ ($i\in[m], j\in[s_i]$) as input, 
	Algorithm~\ref{alg:greedy-kmax} can output a subset $S $ such that $r_D(S) \ge (1-1/e) \cdot \OPT$.
\end{lemma}

\begin{proof}
	For any $x\in[0, 1]^m$, let $f_x(S) = \max_{i\in S}x_i$ be a set function defined on $2^{[m]}$. (Define $f_x(\emptyset) = 0$.) We can verify that $f_x(S)$ is monotone and submodular:
	\begin{itemize}
		\item \emph{Monotonicity.}
		For any $A \subseteq B \subseteq [m]$, we have $f_x(A) = \max_{i\in A}x_i \le \max_{i\in B}x_i = f_x(B)$.
		\item \emph{Submodularity.}
		For any $A \subseteq B \subseteq [m]$ and any $k\in [m]\setminus B$, there are three cases (note that $\max_{i\in A}x_i \le \max_{i\in B}x_i$):
		\begin{enumerate}[(i)]
			\item If $x_k \le \max_{i\in A} x_i $, then $f_x(A\cup\{k\}) - f_x(A) = 0 = f_x(B\cup\{k\}) - f_x(B)$.
			\item If $\max_{i\in A} x_i < x_k \le \max_{i\in B} x_i$, then $f_x(A\cup\{k\}) - f_x(A) = x_k - \max_{i\in A} x_i > 0 = f_x(B\cup\{k\}) - f_x(B)$.
			\item If $x_k > \max_{i\in B} x_i $, then $f_x(A\cup\{k\}) - f_x(A) = x_k - \max_{i\in A} x_i \ge x_k - \max_{i\in B} x_i  = f_x(B\cup\{k\}) - f_x(B)$.
		\end{enumerate}
		Therefore, we always have $f_x(A\cup\{k\}) - f_x(A) \ge f_x(B\cup\{i\}) - f_x(B)$. The function $f_x(S)$ is submodular.
	\end{itemize}
	
	For any $S \subseteq [m]$ we have
	\begin{equation*}
	r_D(S) = \sum_{j_1 = 1}^{s_1} \sum_{j_2 = 1}^{s_2} \cdots \sum_{j_m = 1}^{s_m} f_{(v_{1, j_1}, \ldots, v_{m, j_m})}(S) \prod_{i=1}^{m} p_{i, j_i}.
	\end{equation*}
	Since each set function $f_{(v_{1, j_1}, \ldots, v_{m, j_m})}(S)$ is monotone and submodular,
	$r_D(S)$ is a convex combination of monotone submodular functions on $2^{[m]}$. Therefore, $r_D(S)$ is also a monotone submodular function.
	According to the classical result on submodular maximization \cite{nemhauser1978submod}, the greedy algorithm can find a $(1-1/e)$-approximate solution to $\max_{S \subseteq [m], |S|\le K}\{r_D(S)  \}$.
	%Moreover, if for any $S\in \calF$, $r_D(S)$ can be calculated in polynomial time, then the greedy algorithm also uses polynomial time.
\end{proof}

%!TEX root=nips2016_cmabgeneral.tex

\newcommand{\tX}{\tilde{X}}
\newcommand{\tY}{\tilde{Y}}
\newcommand{\tZ}{\tilde{Z}}

\newcommand{\Sg}{\mathsf{Sig}}
\newcommand{\sg}{\mathsf{sg}}
\newcommand{\SG}{\mathsf{SG}}

\newcommand{\support}{\mathsf{supp}}
\newcommand{\copt}{\mathsf{W}}
\newcommand{\IDX}{\mathsf{IDX}}
\newcommand{\DS}{\mathsf{DS}}
\newcommand{\PV}{\mathsf{Val}}

\newcommand{\eat}[1]{}

\subsection{PTAS}
Now we provide a PTAS for the $K$-MAX problem.
In other words, we give an algorithm which,
given any fixed constant $0<\varepsilon<1/2$,
can find a solution $S$ of cardinality $|K|$
such that $r_D(S)\geq (1-\varepsilon) \cdot \OPT$
in polynomial time.

We first provide an overview of our approach, and then spell out the details later.
\begin{enumerate}
	\item (Discretization) We first transform each $X_i$ to another
	discrete distribution $\tX_i$, such that all $\tX_i$'s are 
	supported on a set of size $O(1/\varepsilon^2)$.
	\item (Computing signatures)
	For each $X_i$, we can compute from $\tX_i$ 
	a signature $\Sg(X_i)$ which is a vector of size $O(1/\varepsilon^2)$.
	For a set $S$, we define its signature $\Sg(S)$ to be $\sum_{i\in S} \Sg(X_i)$.
	We show that if two sets $S_1$ and $S_2$ have the same signature,
	their objective values are close (Lemma~\ref{lm:sig}).
	\item (Enumerating signatures)
	We enumerate all possible signatures (there are polynomial number of them when treating $\varepsilon$
	as a constant)
	and try to find the one 
	which is the signature of a set of size $K$,
	and the objective value is maximized.
\end{enumerate}

\subsubsection{Discretization}

We first describe the discretization step.
We say that a random variable $X$ follows the Bernoulli distribution
$B(v, q)$ if $X$ takes value $v$ with probability $q$ and value $0$ with probability $1 - q$.
%Denote $s(X) = s$ and $p(X) = q$ for $X = B(s, q)$.
For any discrete distribution, we can rewrite it as the maximum of a set of Bernoulli 
distributions.

\begin{definition}
	Let $X$ be a discrete random variable with support $\{v_1, v_2, \ldots, v_s\} (v_1 < v_2 < \cdots < v_s)$ and $\Pr[X = v_{j}] = p_j$.
We define a set of independent Bernoulli random variables $\{Z_j\}_{j\in [s]}$ as
\[
Z_j \sim B\left(v_j\,\, , \frac{p_j}{\sum_{j' \leq j} p_{j'}}\right).
\]
We call  $\{Z_j\}$ the Bernoulli decomposition of $X_i$. 
\end{definition}

\begin{lemma}
For a discrete distribution $X$ and its Bernoulli decomposition $\{Z_j\}$, 
$\max_{j} \{Z_j\}$ has the same distribution with $X$.
\end{lemma}

\begin{proof}
	We can easily see the following:
\begin{align*}
\Pr[\max_j \{Z_j\} = v_i] & = \Pr[Z_i = v_i]\prod_{i' > i} \Pr[Z_{i'} = 0] \\
&= 
\frac{p_i}{\sum_{i'\leq i} p_{i'}}\prod_{h > i} \left(1-\frac{p_h}{\sum_{h'\leq h} p_{h'}}\right) \\
& = \frac{p_i}{\sum_{i'\leq i} p_{i'}}\prod_{h > i} \frac{\sum_{h'\leq h - 1} p_{h'}}{\sum_{h'\leq h} p_{h'}} 
= p_i.
\end{align*}
Hence, $\Pr[\max_j \{Z_j\} = v_i]=\Pr[X=v_i]$ for all $i \in [s]$. 
\end{proof}

\begin{algorithm}[t]
	\caption{Discretization}
	\label{Alg:Disc}
	\begin{algorithmic}[1]
		\STATE We first run \texttt{Greedy-K-MAX} to obtain a solution $S_G$ and let $\copt = r_D(S_G)$.
		\FOR {$i=1$ \TO $m$}
		\STATE Compute the Bernoulli decomposition $\{Z_{i,j}\}_j$ of $X_i$.
		 \FORALL{$Z_{i, j}$}
		\STATE %\textbf{for each} $Z_{i,j}$, 
		Create another Bernoulli variable $\tilde Z_{i,j}$
		as follows:
		\IF {$v_{i,j} > \copt/\varepsilon$}
		\STATE  
		Let $\tilde Z_{i,j} \sim B\left(\frac{\copt}{\varepsilon},  \E[Z_{i,j}]\frac{\varepsilon}{\copt}\right)$ (Case 1)
		\ELSE 
		\STATE  
		Let $\tilde Z_{i,j} = \lfloor\frac{Z_{i,j}}{\varepsilon \copt}\rfloor \varepsilon \copt$ (Case 2)
		\ENDIF
%		\STATE Define set $\IDX_k = \{j | k \leq \frac{s(\tilde Z_{i,j})}{\varepsilon \copt} < (k+1)\}$. Let $Y_{i,k} = \max_{j\in \IDX_k} \lfloor \frac{\tilde Z_{i,j}}{\varepsilon \copt}\rfloor \varepsilon \copt$ for $k = 0, 1, ..., z = \lfloor \varepsilon^{-2} \rfloor$, which follows a Bernoulli distribution $B(k\varepsilon \copt, q_{i,k})$ with $q_{i,k} = 1 - \prod_{j \in \IDX_k} (1- p(\tilde Z_{i,j}))$
		\ENDFOR
		\STATE Let $\tX_i=\max_j \{\tilde Z_{ij}\}$
		\ENDFOR
	\end{algorithmic}
\end{algorithm}

Now, we describe how to construct the discretization $\tX_i$ of $X_i$ for all $i\in [m]$.
The pseudocode can be found in Algorithm~\ref{Alg:Disc}.
We first run \texttt{Greedy-K-MAX} to obtain a solution $S_G$. 
Let $\copt = r_D(S_G)$.
By Lemma~\ref{lem:max_const}, we know that $\copt\geq (1-1/e)\OPT$.
Then we compute the Bernoulli decomposition $\{Z_{i,j}\}_j$ of $X_i$.
For each $Z_{i,j}$, we create another Bernoulli variable $\tilde Z_{i,j}$
as follows:
Recall that $v_{i,j}$ is the nonzero possible value of $Z_{ij}$. 
We distinguish two cases.
Case 1: 
If $v_{i,j} > \copt/\varepsilon$, then
we let $\tilde Z_{i,j} \sim B\left(\frac{\copt}{\varepsilon}, \E[Z_{i,j}]\frac{\varepsilon}{\copt}\right)$.
It is easy to see that $\E[\tZ_{ij}]=\E[Z_{ij}]$.
Case 2:
If $v_{i,j} \leq \copt/\varepsilon$,  then
we let $\tilde Z_{i,j} = \lfloor\frac{Z_{i,j}}{\varepsilon \copt}\rfloor \varepsilon \copt$.
We note that 
more than one $\tZ_{ij}$'s may have the same support, and 
all $\tZ_{ij}$'s are supported on 
$\DS=\{0,\varepsilon\copt, 2\varepsilon\copt,\ldots, \copt/\varepsilon\}$.
Finally, we let $\tX_i=\max_j \{\tilde Z_{ij}\}$,
which is the discretization of $X_i$.
Since $\tX_i$ is the maximum of a set of Bernoulli distributions,
it is also a discrete distribution supported on $\DS$. 
We can easily compute $\Pr[\tX_i=v]$ for any $v\in \DS$.

Now, we show that the discretization only incurs a small loss in 
the objective value.
The key is to show that we do not lose much in 
the transformation from $Z_{i, j}$'s to $\tZ_{i, j}$'s. 
We prove a slightly more general lemma as follows.

\begin{lemma}
\label{bound}
Consider any set of Bernoulli variables $\{Z_i \sim B(a_i, p_i)\}_{1\le i \le n}$.
Assume that $\E[\max_{i\in [n]} Z_i] < c\copt$, where $c$ is a constant
such that $c\varepsilon<1/2$.
For each $Z_i$, we create a Bernoulli variable $\tZ_i$ in the same way as 
Algorithm~\ref{Alg:Disc}.
Then the following holds:
\[
\E[\max Z_i] \geq \E[\max \tilde Z_i] \geq \E[\max Z_i]-(2c+1)\varepsilon\copt.
\]
\end{lemma}

\begin{proof}
Assume $a_1$ is the largest among all $a_i$'s.

If $a_1 < \copt/\varepsilon$, all $\tilde Z_i$ are created in Case 2.
In this case, it is obvious to have that
\[
\E[\max Z_i] \geq \E[\max \tilde Z_i] 
\geq \E[\max Z_i]-\varepsilon\copt.
\]

If $a_1\geq \copt/\varepsilon$, the proof is slightly more complicated.
Let $L=\{i \mid a_i\geq \copt/\varepsilon \}$.
We prove by induction on $n$
(i.e., the number of the variables) the following more general claim:
\begin{align}
\label{eq:claim}
\E[\max Z_i] \geq 
\E[\max \tilde Z_i] \geq \E[\max Z_i]-\varepsilon\copt - c\sum_{i\in L} \varepsilon a_ip_i.
\end{align}

Consider the base case $n=1$.
The lemma holds immediately in Case 1 as $\E[Z_1] = \E[\tilde{Z_1}]$.

Assuming the lemma is true for $n = k$,
we show it also holds for $n = k + 1$.
Recall 
we have $\tilde Z_1 \sim  B(\frac{\copt}{\varepsilon}, \varepsilon \E[Z_1]/\copt)$. 
Thus
\begin{align*}
\E[\max_{i\geq 1} Z_i] - \E[\max_{i\geq 1} \tilde Z_i] 
=& a_1p_1 + (1 - p_1)\E[\max_{i\geq 2} Z_i] - a_1p_1 - (1 - \varepsilon \E[Z_1]/\copt)\E[\max_{i\geq 2} \tilde Z_i]\\
\geq& (1 - p_1)\E[\max_{i\geq 2} \tZ_i] - (1 - \varepsilon \E[Z_1]/\copt)\E[\max_{i\geq 2} \tZ_i]\\
=& (\varepsilon a_1p_1/\copt - p_1)\E[\max_{i\geq 2} \tZ_i]
\geq 0,
\end{align*}
where the first inequality follows from the induction hypothesis and 
the last from $a_1\geq \copt/\varepsilon$.
The other direction can be seen as follows:
\begin{align*}
\E[\max_{i\ge1} \tilde Z_i] - \E[\max_{i\ge1} Z_i]
= & a_1p_1 + (1 - \varepsilon \E[Z_1]/\copt)\E[\max_{i\geq 2} \tilde Z_i] - 
(a_1p_1 + (1 - p_1)\E[\max_{i\geq 2} Z_i])\\
\geq &  (1 - \varepsilon \E[Z_1]/\copt)\E[\max_{i\geq 2} Z_i] - (1 - p_1)\E[\max_{i\geq 2}  Z_i] -\varepsilon \copt- c\sum_{i\in L\setminus\{1\}} \varepsilon a_ip_i\\
\geq & (- \varepsilon \E[Z_1]/\copt)\E[\max_{i\geq 2}  Z_i]
-\varepsilon \copt- c\sum_{i\in L\setminus\{1\}} \varepsilon a_ip_i\\
\geq & -\varepsilon\copt - c\sum_{i\in L} \varepsilon a_ip_i,
\end{align*}
where the last inequality holds since $\E[\max_{i\geq 2}  Z_i]\leq c\copt$. 
This finishes the proof of \eqref{eq:claim}.

Now, we show that $\sum_{i\in L}  a_ip_i\leq 2\copt$.
This can be seen as follows.
First, we can see from Markov inequality that 
\[
\Pr[\max Z_i > \copt/\varepsilon] \leq c\varepsilon. 
\]
Equivalently, we have $\prod_{i\in L}(1-p_i) \geq 1-c\varepsilon$.
Then, we can see that
\begin{align*}
\copt \geq \sum_{i\in L} a_i \prod_{j<i} (1-p_j) p_i
\geq (1-c\varepsilon) \sum_{i\in L}a_i p_i \geq \frac{1}{2} \sum_{i\in L}a_i p_i.
\end{align*}
Plugging this into \eqref{eq:claim}, we prove the lemma.
\end{proof}

\eat{
\begin{lemma}
After the discretization, \[\E[\max_i X_i] \geq \E[\max_{i,k} Y_{i,k}] \geq \E[\max_i X_i] - O(\varepsilon)\copt.\]
\end{lemma}
\begin{proof}
First, for any $\tilde Z_{i,j}$, $\frac{s(\tilde Z_{i,j})}{\varepsilon \copt} \leq \frac{\copt /\varepsilon}{\varepsilon\copt} < z + 1$, $j$ must be in some index set $\IDX_k$. Hence $\E[\max_{i,j} \lfloor \frac{\tilde Z_{i,j}}{\varepsilon \copt}\rfloor \varepsilon \copt]= \E[\max_{i,k} Y_{i,k}]$.

We have $\forall S$, $\E[\max_{i\in S} X_i] \leq OPT \leq \copt/\alpha$, where $\alpha = 1 - 1/e$. Now we let $c = 1/\alpha$ and by Lemma~\ref{bound} we have:
\[\E[\max_i X_i] = \E[\max_{i,j} Z_{i,j}]\geq \E[\max_{i,j} \tilde Z_{i,j}] \geq \E[\max_{i,j} \lfloor \frac{\tilde Z_{i,j}}{\varepsilon \copt}\rfloor \varepsilon \copt]= \E[\max_{i,k} Y_{i,k}],\]
and
\[\E[\max_{i,j} \tilde Z_{i,j}] \geq (1 - c\varepsilon)\E[\max_{i,j} Z_{i,j}] = (1 - c\varepsilon)\E[\max_i X_i].\]

On the other hand, say $A = \{\tilde Z_{i,j} = s_{i,j}\}$ is a realization of $\{\tilde Z_{i,j}\}$. Define $V[A] = \max_{i,j} s_{i,j}$ and $V'[A] = \max_{i,j}\{\lfloor \frac{s_{i,j}}{\varepsilon \copt}\rfloor\varepsilon \copt\}$. Thus we have
\begin{align*}
&\E[\max_{i,j} \tilde Z_{i,j}] - \E[\max_{i,k} Y_{i,k}]\\
=& \E[\max_{i,j} \tilde Z_{i,j}] - \E[\max_{i,j} \lfloor \frac{\tilde Z_{i,j}}{\varepsilon \copt}\rfloor \varepsilon \copt]\\
=& \sum_{A}V[A]\Pr[A] - \sum_{A}V'[A]\Pr[A]\\
=&\sum_{A}(V[A] - V'[A])\Pr[A]\\
\leq& \max_{A} |V[A] - V'[A]|\\
\leq& \varepsilon \copt
\end{align*}
Therefore 
\[\E[\max_{i,k} Y_{i,k}] \geq \E[\max_{i,j} \tilde Z_{i,j}] - \varepsilon \copt \geq (1 - c\varepsilon)\E[\max_i X_i] -\varepsilon \copt= \E[\max_i X_i] - O(\varepsilon)\copt.\]
\end{proof}
}

\begin{corollary}
	\label{lm:bound2}
	For any set $S \subseteq [m]$, suppose $\E[\max_{i\in S} X_i] < c\copt$, 
	where $c$ is a constant
	such that $c\varepsilon<1/2$.
	Then the following holds:
	\[
	\E[\max_{i\in S} X_i] \geq \E[\max_{i\in S} \tilde X_i] \geq 
	\E[\max_{i\in S} X_i]-(2c+1)\varepsilon\copt.
	\]
\end{corollary}

\subsubsection{Signatures}
\eat{
Now for each $X_i$ we have an approximate Bernoulli decomposition $\{Y_{i,k}\}_{k=1}^z$ with $Y_{i,k} = B(k\varepsilon \copt, q_{i,k})$. we define the signature for variable $X_i$ and set $S$ respectively:
\[
Sg(X_i) = (\lfloor \frac{-\log{(1 - q_{i,1})}}{\varepsilon^4/n}\rfloor \frac{\varepsilon^4}{n}, \ldots, \lfloor \frac{-\log{(1 - q_{i,z - 1})}}{\varepsilon^4/n}\rfloor \frac{\varepsilon^4}{n}),
\]
and
\[
Sg(S) = \sum_{i\in S} Sg(X_i).
\]
}
For each $X_i$, we have created its discretization 
$\tX_i=\max_j \{\tilde Z_{ij}\}$.
Since $\tX_i$ is a discrete distribution, we can define its
Bernoulli decomposition $\{Y_{ij}\}_{j\in [h]}$ where $h=|\DS|$.
Suppose $Y_{ij} \sim B(j\varepsilon\copt, q_{ij})$.
Now, we define the signature of $X_i$ to be the vector
$
\Sg(X_i) = (\Sg(X_i)_1,\ldots, \Sg(X_i)_h)
$ 
where
\[
\Sg(X_i)_j=
\min \left(\left\lfloor \frac{-\ln{(1 - q_{ij})}}{\varepsilon^4/m}\right\rfloor, 
\,\,
\left\lfloor \frac {\ln(1/\varepsilon^4)}{\varepsilon^4/m}\right\rfloor\right)
\cdot\frac{\varepsilon^4}{m}
\qquad
 j\in [h].
\]
For any set $S$, define its signature to be 
\[
\Sg(S)=\sum_{i\in S} \Sg(X_i).
\]

Define the set $\SG$ of {\em signature vectors} to be all nonnegative $h$-dimensional vectors, where each coordinate is an integer multiple of 
$\varepsilon^4/m$ and at most $m\ln(1/\varepsilon^4)$.
Clearly, the size of $\SG$ is $O\left(
\left(m \varepsilon^{-4}\log(h/\varepsilon^2)\right)^{h-1}
\right)=\tilde O(m^{O(1/\varepsilon^2)})$, which is  polynomial for any fixed constant $\varepsilon>0$ (recall $h=|\DS|=O(1/\varepsilon^2)$).

Now, we prove the following crucial lemma.
\begin{lemma}
\label{lm:sig}
Consider two sets $S_1$ and $S_2$.
If $\Sg(S_1) = \Sg(S_2)$, the following holds:
\[
\left|\E[\max_{i\in S_1} \tX_i] - \E[\max_{i\in S_2} \tX_i]\right| \leq O(\varepsilon)\copt.
\]
\end{lemma}
\begin{proof}
Suppose $\{Y_{ij}\}_{j\in [h]}$ is the Bernoulli decomposition of $\tX_i$.
For any set $S$,
we define $Y_k(S) = \max_{i\in S} Y_{ik}$
(it is the max of a set of Bernoulli distributions). 
It is not hard to see that $Y_k(S)$
has a Bernoulli distribution $B(k\varepsilon \copt, p_k(S))$ 
with $p_k(S) = 1 - \prod_{i\in S}(1 - q_{ik})$. 
As $\Sg(S_1) = \Sg(S_2)$, we have that
\begin{align*}
|p_k(S_1) - p_k(S_2)| 
&=|\prod_{i\in S_1}(1-q_{ik})-\prod_{i\in S_2}(1-q_{ik})| \\
&=
\left|\exp\left(\sum_{i\in S_1}\ln(1-q_{ik})\right)
-\exp\left(\sum_{i\in S_2}\ln(1-q_{ik})\right)\right| \\
&\leq 2\varepsilon^4 \qquad\forall k\in [h].
\end{align*}
Noticing $\max_{i\in S}\tX_i=\max_k Y_k(S)$, we have that
\begin{align*}
\left|\E[\max_{i\in S_1} \tX_i] - \E[\max_{i\in S_2} \tX_i]\right| 
=& \left|\E[\max_k Y_k(S_1)] - \E[\max_k Y_k(S_2)]\right| \\
\leq & \frac\copt\varepsilon \left(\sum_k |p_k(S_1) - p_k(S_2)| \right)\\
\leq & 4h\varepsilon^{3} \copt = O(\varepsilon)\copt
\end{align*}
where the first inequality follows from 
Lemma~\ref{lem:l1-product-distribuction}.
\end{proof}

For any signature vector $\sg$, we associate to it a set of random variables
$\{B_k \sim B(k\varepsilon \copt, 1 - e^{-\sg_k})\}_{k = 1}^h$.\footnote{
	It is not hard to see the signature of $\max_{k \in[h] } B_k$ is exactly $\sg$.
}  
Define the value of $\sg$ to be $\PV(\sg) = \E[\max_{k \in [h]} B_k]$.

\begin{corollary}
\label{bound:PV}
For any feasible set $S$ with $\Sg(S) = \sg$, 
$|\E[\max_{i\in S} \tX_i] - \PV(\sg)| \leq O(\varepsilon)\copt$.
Moreover, combining with Corollary~\ref{lm:bound2},
we have that $|\E[\max_{i\in S} X_i] - \PV(\sg)| 
\leq O(\varepsilon)\copt.$
\end{corollary}

\eat{
\begin{corollary}
\label{mono}
For any set $S$ with $\Sg(S) \geq \sg$ (entrywise), $\E[\max_{i\in S} \tX_i] \geq \PV(\sg) - O(\varepsilon)\copt$.
Hence,
combining with Corollary~\ref{lm:bound2}, $\E[\max_{i\in S} X_i] \geq \PV(\sg) - O(\varepsilon)\copt$.
\end{corollary}
\begin{proof}
It is easy to verify that
if $\sg_1\geq \sg_2$ (entrywise),
$\PV(\sg_2)\geq \PV(\sg_2)$.
Using Corollary~\ref{bound:PV}, we know that 
$\E[\max_{i\in S} \tX_i]\geq \PV(\Sg(S))-O(\varepsilon)$.
The lemma follows immediately.
\end{proof}
}

\subsubsection{Enumerating Signatures}

Our algorithm enumerates all signature vectors $\sg$ in $\SG$.
For each $\sg$, we check if we can find a set $S$ of size $K$
such that $\Sg(S)= \sg$. 
This can be done by a standard dynamic program in 
$\tilde O(m^{O(1/\varepsilon^2)})$ time as follows:
We use Boolean variable $R[i][j][\sg']$ to represent whether signature vector
$\sg'\in \SG$ can be dominated by $i$ variables 
in set $\{X_1, \ldots, X_j\}$.
The dynamic programming recursion is 
\[
R[i][j][\sg']= R[i][j - 1][\sg'] \wedge R[i-1][j - 1][\sg'-\Sg(X_j)].
\]
%where $(x)_+$ represents the vector which replaces all negative entries of $x$ by 0.

If the answer is yes (i.e., we can find such $S$),
we say $\sg$ is a feasible signature vector and $S$ is a candidate set.
Finally, we pick the candidate set with maximum $r_D(S)$
and output the set.
The pseudocode can be found in Algorithm~\ref{Alg:PTAS}.

\begin{algorithm}[t]
	\caption{\texttt{PTAS-K-MAX}}
	\label{Alg:PTAS}
	\begin{algorithmic}[1]
		\STATE $U \leftarrow \emptyset$
	\FORALL {signature vector $\sg\in \SG$}
		\STATE Find a set $S$ such that $|S|=K$ and 
		$\Sg(S) = \sg$
		\IF {$r_D(S) > r_D(U)$} 
		\STATE  $U \leftarrow S$
		\ENDIF
		\ENDFOR
		%\STATE Output $U$ as the final solution
		\ENSURE $U$
	\end{algorithmic}
\end{algorithm}

Now, we are ready to prove Theorem~\ref{thm:kmax} by showing Algorithm~\ref{Alg:PTAS} is a PTAS for the $K$-MAX problem.

\begin{proof}[Proof of Theorem~\ref{thm:kmax}]
Suppose $S^*$ is the optimal solution 
and $\sg^*$ is the signature of $S^*$. 
%If any coordinates of $\sg^*$ exceed $\log\frac{z}{\varepsilon^2}$, round down to %$\log\frac{z}{\varepsilon^2}$.
%Denote the new signature vector as $\sg$.
%It is easy to verify that 
By Corollary~\ref{bound:PV}, we have that
$|\OPT-\PV(\sg^*)|\leq O(\varepsilon)\copt$. 

When Algorithm~\ref{Alg:PTAS} is enumerating $\sg^*$,
it can find a set $S$ such that $\Sg(S)=\sg^*$ (there exists at 
least one such set since $S^*$ is one). 
Therefore, we can see that 
\[
|\E[\max_{i\in S} X_i] - \E[\max_{i\in S^*} X_i]| 
\leq |\PV(\sg^*) - \max_{i\in S} X_i| + |\PV(\sg^*) - \E[\max_{i\in S^*} X_i]| \leq O(\varepsilon) \copt.
\] 
Let $U$ be the output of  Algorithm~\ref{Alg:PTAS}.
Since $\copt\geq (1-1/e)\OPT$,
we have $r_D(U)\geq r_D(S) = \E[\max_{i\in S} X_i] \geq (1 - O(\varepsilon))\OPT.$

The running time of the algorithm is 
polynomial for a fixed constant $\varepsilon>0$, since the number of signature vectors is polynomial 
and the dynamic program in each iteration also runs in 
polynomial time. Hence, we have a PTAS for the $K$-MAX
problem. 
\end{proof}

\paragraph{Remark.}
In fact, Theorem~\ref{thm:kmax} can be generalized in the following way:
instead of the cardinality constraint $|S|\leq K$, 
we can have more general combinatorial constraint on the feasible set $S$.
As long as we can execute line 3 in Algorithm~\ref{Alg:PTAS} in polynomial time,
the analysis wound be the same.
Using the same trick as in~\cite{LiD11},
we can extend the dynamic program to a more general class of combinatorial 
constraints
where there is a pseudo-polynomial time for the exact version\footnote{
	In the exact version of a problem, we ask for a feasible set $S$
	such that total weight of $S$ is exactly a given target value $B$.
	For example, in the exact spanning tree problem where each edge has
	an integer weight, we would like to find
	a spanning tree of weight exactly $B$.
	}
of the deterministic version of the corresponding problem.
The class of constraints includes $s$-$t$ simple paths, knapsacks, spanning trees, matchings, etc.

\section{Empirical Comparison between the SDCB Algorithm and Online Submodular Maximization on the $K$-MAX Problem} \label{appdx:experiment}

We perform experiments to compare the SDCB algorithm with the online submodular maximization algorithm in \cite{Streeter2008}, on the $K$-MAX problem.

\paragraph{Online Submodular Maximization.}
First we briefly describe the online submodular maximization problem considered in \cite{Streeter2008} and the algorithm therein.
At the beginning, an oblivious adversary sets a sequence of submodular functions $f_1, f_2, \ldots, f_T$ on $2^{[m]}$, where $f_t$ will be used to determine the reward in the $t$-th round.
In the $t$-th round, if the player selects a feasible super arm $S_t$, the reward will be $f_t(S_t)$.
This model covers the $K$-MAX problem as an instance: suppose $X^{(t)} = (X^{(t)}_1, \ldots, X^{(t)}_m) \sim D$ is the outcome vector sampled in the $t$-th round, then the function $f_t(S) = \max_{i\in S} X^{(t)}_i$ is submodular and will determine the reward in the $t$-th round.
We summarize the algorithm in Algorithm~\ref{alg:online-submod-max}. It uses $K$ copies of the \texttt{Exp3} algorithm (see \cite{auer2002nonstochastic} for an introduction).
For the $K$-MAX problem, Algorithm~\ref{alg:online-submod-max} achieves an $O(K \sqrt{mT \log m})$ upper bound on the $(1-1/e)$-approximation regret.

\begin{algorithm}[t]
	\caption{Online Submodular Maximization \cite{Streeter2008}}  
	\label{alg:online-submod-max}
	\begin{algorithmic}[1]
		%\REQUIRE $m$, $\calF$, $\supp(D_i) = \{v_{i, 1}, v_{i, 2}, \ldots, v_{i, s_i}\}$ for each $i\in[m]$
		%\vspace{3pt}
		\STATE Let $\A_1, \A_2, \ldots, \A_K$ be $K$ instances of \texttt{Exp3}
		\vspace{3pt}
		
		\FOR{$t=1, 2, \ldots$}
		\STATE // Action in the $t$-th round
		%\FOR{$i=1, 2, \ldots, m$}
		\FOR{$i=1$ \TO $K$}
		\STATE Use $\A_i$ to select an arm $a_{t, i} \in [m]$
		\ENDFOR
		\STATE Play the super arm $S_t \leftarrow \bigcup_{i=1}^K \{a_{t, i} \}$
		\FOR{$i=1$ \TO $K$}
		\STATE Feed back $f_t (\bigcup_{j=1}^i \{a_{t, j} \} ) - f_t (\bigcup_{j=1}^{i-1} \{a_{t, j} \} )$ as the payoff $\A_i$ receives for choosing $a_{t, i}$
		\ENDFOR 
		\ENDFOR
	\end{algorithmic}
\end{algorithm}

\paragraph{Setup.}
We set $m=9$ and $K=3$, i.e., there are $9$ arms in total and it is allowed to select at most $3$ arms in each round.
We compare the performance of \texttt{SDCB}/\texttt{Lazy-SDCB} and the online submodular maximization algorithm on four different distributions.
Here we use the greedy algorithm \texttt{Greedy-K-MAX} (Algorithm~\ref{alg:greedy-kmax}) as the offline oracle.

Let $X_i \sim D_i$ ($i=1, \ldots, 9$). We consider the following distributions.
For all of them, the optimal super arm is $S^* = \{1, 2, 3\}$.
\begin{itemize}
	\item Distribution 1:
	All $D_i$'s have the same support $\{0, 0.2, 0.4, 0.6, 0.8, 1\}$.
	
	For $i\in\{1, 2, 3\}$, $\Pr[X_i = 0] = \Pr[X_i = 0.2] = \Pr[X_i = 0.4] = \Pr[X_i = 0.6] = \Pr[X_i = 0.8] = 0.1$ and $\Pr[X_i = 1] = 0.5$.
	
	For $i\in\{4, 5, 6, \ldots, 9\}$, $\Pr[X_i = 0] = 0.5$ and $ \Pr[X_i = 0.2] = \Pr[X_i = 0.4] = \Pr[X_i = 0.6] = \Pr[X_i = 0.8] = \Pr[X_i = 1] = 0.1$.
	
	\item Distribution 2:
	All $D_i$'s have the same support $\{0, 0.2, 0.4, 0.6, 0.8, 1\}$.
	
	For $i\in\{1, 2, 3\}$, $\Pr[X_i = 0] = \Pr[X_i = 0.2] = \Pr[X_i = 0.4] = \Pr[X_i = 0.6] = \Pr[X_i = 0.8] = 0.1$ and $\Pr[X_i = 1] = 0.5$.
	
	For $i\in\{4, 5, 6, \ldots, 9\}$, $\Pr[X_i = 0] = \Pr[X_i = 0.2] = \Pr[X_i = 0.4] = \Pr[X_i = 0.6] = \Pr[X_i = 0.8] = 0.12$ and $\Pr[X_i = 1] = 0.4$.
	
	\item Distribution 3:
	All $D_i$'s have the same support $\{0, 0.2, 0.4, 0.6, 0.8, 1\}$.
	
	For $i\in\{1, 2, 3\}$, $\Pr[X_i = 0] = \Pr[X_i = 0.2] = \Pr[X_i = 0.4] = \Pr[X_i = 0.6] = \Pr[X_i = 0.8] = 0.1$ and $\Pr[X_i = 1] = 0.5$.
	
	For $i\in\{4, 5, 6\}$, $\Pr[X_i = 0] = \Pr[X_i = 0.2] = \Pr[X_i = 0.4] = \Pr[X_i = 0.6] = \Pr[X_i = 0.8] = 0.12$ and $\Pr[X_i = 1] = 0.4$.
	
	For $i\in\{7, 8, 9\}$, $\Pr[X_i = 0] = \Pr[X_i = 0.2] = \Pr[X_i = 0.4] = \Pr[X_i = 0.6] = \Pr[X_i = 0.8] = 0.16$ and $\Pr[X_i = 1] = 0.2$.
	
	\item Distribution 4:
	All $D_i$'s are continuous distributions on $[0, 1]$.
	
	For $i\in\{1, 2, 3\}$, $D_i$ is the uniform distribution on $[0, 1]$.
	
	For $i\in\{4, 5, 6, \ldots, 9\}$, the probability density function (PDF) of $X_i$ is
	\[
	f(x) = \begin{cases}
	1.2 & \quad x \in [0, 0.5],\\
	0.8 & \quad x \in (0.5, 1].
	\end{cases}
	\]
\end{itemize}

These distributions represent several different scenarios.
Distribution 1 is relatively ``easy'' because the suboptimal arms 4-9's distribution is far away from arms 1-3's distribution,
% are likely to produce much lower outcome than arms 1-3, 
 whereas distribution 2 is ``hard'' since the distribution of arms 4-9 is close to the distribution of arms 1-3.
In distribution 3, the distribution of arms 4-6 is close to the distribution of arms 1-3's, while arms 7-9's distribution is further away.
Distribution 4 is an example of a group of continuous distributions for which \texttt{Lazy-SDCB} is more efficient than \texttt{SDCB}.

We use \texttt{SDCB} for distributions 1-3, and \texttt{Lazy-SDCB} (with known time horizon) for distribution 4.
Figure \ref{fig} shows the regrets of both \texttt{SDCB} and the online submodular maximization algorithm.
We plot 
 the $1$-approximation regrets instead of the $(1-1/e)$-approximation regrets, since the greedy oracle usually performs much better than its $(1-1/e)$-approximation guarantee.
We can see from Figure~\ref{fig} that our algorithms achieve much lower regrets in all examples.

\begin{figure}[t] 
	\centering
	\subfigure[Distribution 1] { \label{fig:dist1}     
		\includegraphics[width=0.48\textwidth]{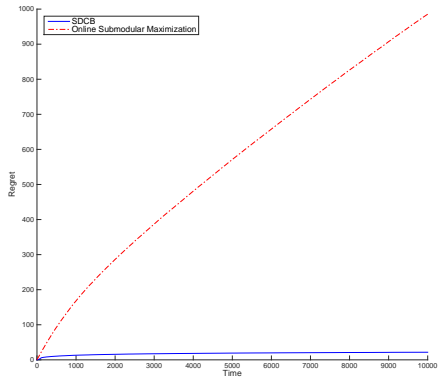}  
		}
		\subfigure[Distribution 2] { \label{fig:dist2}     
			\includegraphics[width=0.48\textwidth]{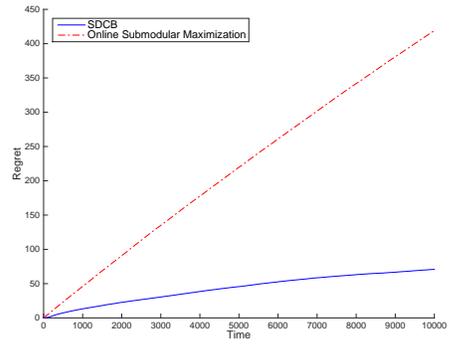} 
			}
			\subfigure[Distribution 3] { \label{fig:dist3}     
				\includegraphics[width=0.48\textwidth]{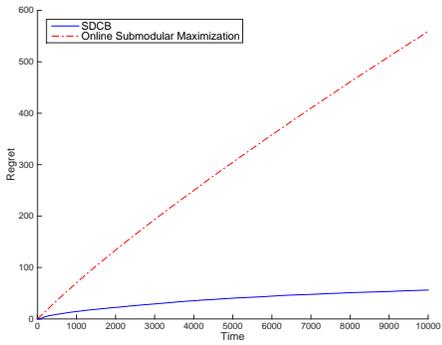}  
				}
				\subfigure[Distribution 4] { \label{fig:dist4}     
					\includegraphics[width=0.48\textwidth]{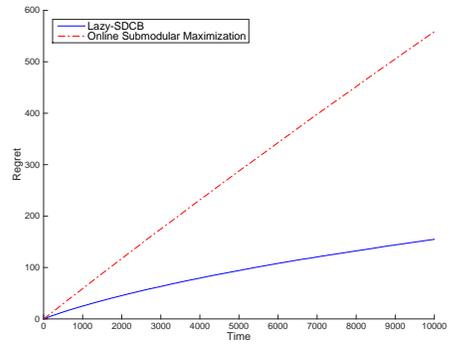}  
				}
				\caption{Regrets of \texttt{SDCB}/\texttt{Lazy-SDCB} and Algorithm~\ref{alg:online-submod-max} on the $K$-MAX problem, for  distributions 1-4. The regrets are averaged over $20$ independent runs.}
				\label{fig}
				\end{figure}

\end{document}